\documentclass[journal]{IEEEtran}
\IEEEoverridecommandlockouts
\ifCLASSINFOpdf
\else
\fi
\usepackage{multirow}
\usepackage{pdflscape}
\usepackage{amsmath,amssymb,amsfonts}
\usepackage{graphicx}
\usepackage{xcolor}
\usepackage{subcaption}
\usepackage{mathtools}
\usepackage{cite}
\usepackage{graphicx}
\usepackage{array}
\usepackage{mathdots} 
\usepackage{amsthm}
\theoremstyle{plain}

\newtheorem{theorem}{Theorem}
\usepackage{float}

\def\BibTeX{{\rm B\kern-.05em{\sc i\kern-.025em b}\kern-.08em
    T\kern-.1667em\lower.7ex\hbox{E}\kern-.125emX}}

\begin{document}

\title{Hankel and Toeplitz Rank-1 Decomposition of Arbitrary Matrices with Applications to Signal Direction-of-Arrival Estimation}

\author{Georgios I. Orfanidis,~\IEEEmembership{Student Member~IEEE}, ~Dimitris A. Pados,~\IEEEmembership{Senior Member~IEEE}, \\George Sklivanitis,~\IEEEmembership{Member~IEEE},~and Elizabeth Serena Bentley,~\IEEEmembership{Senior Member~IEEE}
\thanks{G. I. Orfanidis, D. A. Pados, and G. Sklivanitis are with the Center for Connected Autonomy and AI and the Dept. of Electrical Engineering and Computer Science, Florida Atlantic University, Boca Raton, FL 33431 USA (email: \{gorfanidis2021, dpados, gsklivanitis\}@fau.edu).

E. S. Bentley is with the Air Force Research Laboratory, AFRL/RI, Rome, NY 13441 USA (e-mail: \{elizabeth.bentley.3\}@us.af.mil).

This work was supported in part by the National Science Foundation under Grants CNS-2117822 and EEC-2133516 and the Air Force Research Laboratory under Grant FA8750-25-1-1000.

Distribution A. Approved for public release: Distribution Unlimited: AFRL-2026-2349 on 12 May 2026.
}}
\maketitle

\begin{abstract}
We consider the problems of computing the optimal rank-$1$ Hankel and Toeplitz-structured approximation of arbitrary matrices under $L_2$ and $L_1$-norm error. Such problems arise naturally in engineered systems, including the basic few-shot signal Direction-of-Arrival (DoA) estimation problem that is of importance to modern autonomous systems applications.
We develop accurate and computationally efficient structured matrix decomposition algorithms for both formulations and then derive analytically grounded small-sample-support DoA estimators for practical sensing system deployments. The resulting estimators under the $L_2$ and $L_1$ norms are formally shown to be maximum-likelihood optimal under white Gaussian and Laplace noise, respectively. The estimators are further validated through extensive simulation studies and real-world data experiments in few-shot DoA inference.
\end{abstract}


\section{Introduction}
\IEEEPARstart {E}{stimating} and modeling data from noisy values, while balancing estimation accuracy and model complexity, has been at the forefront of signal processing and machine learning research for decades. A fundamental algebraic approach to this problem is low-rank approximation (LRA) where a data matrix constructed from the noisy measurements is approximated by another matrix of low rank. At its core, LRA seeks to recover the underlying subspace that best represents the noise-free data with the matrix rank serving as a direct measure of model complexity. Depending on the assumed underlying model, low-rank approximation (LRA) can be categorized into unstructured low-rank approximation (ULRA), when the model is assumed to be linear and static, and structured low-rank approximation (SLRA), when the model can be nonlinear or dynamic \cite{markovsky2012low}.

Conventionally, ULRA seeks the rank-$k$ matrix that is closest to the data under the $L_2$-norm error criterion. This problem admits a numerical solution via truncated singular value decomposition (SVD) \cite{eckart1936approximation} and is equivalent to the Principal Component Analysis (PCA) problem of maximum $L_2$-norm data projection with as many principal components as the desired rank \cite{golub1996matrix}. 
$L_2$-norm ULRA has played a significant role across a broad spectrum of application domains in signal processing, information retrieval, computer vision, natural language processing, time-series analysis, and signal denoising and classification \cite{markovsky2012low,pearson1901liii,tenenbaum2000global,belkin2003laplacian,van2002optimum}
and  is appealing for its analytical simplicity, computational efficiency, and scalability (new principal directions naturally extend from existing ones). Nonetheless, it is well recognized that the quadratic formulation is highly sensitive to sporadic faulty observations, since the squared residuals assign disproportionate influence to outliers in the data matrix. Formally, outliers refer to data points that significantly deviate from the nominal data subspace and are unlikely to occur under normal system operation. They can infiltrate the data matrix for reasons including transient sensor malfunctions, erroneous data recording, or external interference. In an effort to make $L_2$-norm ULRA more resilient to such outliers, the research community explored alternative formulations based for example on the $L_1$-norm (i.e., absolute errors). Outlier-resistant methods employing direct $L_1$-norm decomposition error have been proposed in this context in \cite{tsagkarakis2016l1, ke2005robust, brooks2013pure} and references therein. 

Many real-world applications involve nonlinear and dynamic systems that, in addition to the rank consideration, impose specific structural constraints on the approximated data matrix. This gives rise to the problem of SLRA where the objective is to approximate a given data matrix by another matrix of low rank that also satisfies an imposed structure. Such structural constraints are essential, as they encode inherent properties of the underlying physical processes or system dynamics. While a variety of matrix structures have been explored in the literature (e.g., Sylvester, quadratic, bilinear, etc.) with applications ranging from conic section fitting to fundamental matrix estimation \cite{zhang1997parameter,ma2004invitation,tomasi1993shape}, Hankel and Toeplitz structures,\footnote{A matrix is called Hankel if each anti-diagonal has elements of constant value. A matrix is called Toeplitz if each diagonal has elements of constant value \cite{golub1996matrix}.} given their direct connection to dynamic linear time-invariant models (LTI), have gained significant attention across multiple domains, such as  system identification, spectral estimation, time-series analysis, and direction-of-arrival (DoA) estimation\cite{fazel2013hankel,laskar2023proposed,sun2021hankel,sengar2021system,orfanidis2022time,10051870,10096647}. 

Enforcing both the low-rank and the structure constraint gives rise to a challenging nonconvex optimization problem. Ignoring the structure constraint, the problem reduces to standard ULRA. Under the $L_2$-norm, the solution is obtained by means of SVD, whereas under the $L_1$-norm, a suboptimal solution is readily available for real-valued matrices \cite{tsagkarakis2016l1}. Conversely, ignoring the rank constraint leads to a convex least-squares fit over the structured set. When the structure is a linear subspace, such as the Hankel or Toeplitz class, the unique solution is the orthogonal projection of the data onto that subspace in the Frobenius inner product sense. For Hankel matrices, this projection corresponds to  averaging along anti-diagonals \cite{golyandina2001analysis}. Under the 
$L_1$-norm, the Hankel approximation is obtained by replacing anti-diagonal entries by their respective median \cite{kalantari2016singular}. 
A common approach to obtain an approximate solution that simultaneously satisfies the low-rank and structure constraints under the $L_2$-norm is the method of alternating projections, widely known as Cadzow’s algorithm with variants \cite{cadzow2002signal,yin2021low}. In its basic form, the algorithm alternates between projections onto the subspace of structured matrices (e.g., Hankel or Toeplitz) and the manifold of low-rank matrices progressively enforcing both constraints. Given its simplicity, Cadzow’s algorithm is commonly used in practice. Nevertheless, its convergence has not been formally established for the general low-rank structured approximation problem. Recently, the authors in \cite{knirsch2021optimal} formally proved that in the rank-1 Hankel approximation setting Cadzow’s algorithm converges to a fixed point that is typically not the optimal solution. 
Notably, basic Singular Spectrum Analysis (SSA) \cite{golyandina2001analysis}, an increasingly popular time-series analysis framework, can be viewed as a single iteration of Cadzow’s algorithm in which the rank constraint is relaxed in favor of a fully satisfied Hankel structure constraint \cite{gillard2010cadzow}. Consequently, iterative basic SSA, often referred to as sequential SSA \cite{golyandina2001analysis}, is equivalent to Cadzow’s algorithm.

In this paper, we focus on rank-$1$ Hankel and Toeplitz approximations, a fundamental special case of the SLRA problem, considering both $L_2$ and $L_1$ norms. Since every Toeplitz matrix can be written as the product on an exchange matrix{\footnote{An exchange matrix is a square matrix with ones on the the antidiagonal and zeros everywhere else.}} and a Hankel matrix, it suffices to direct our attention to rank-1 Hankel decompositions of arbitrary matrices. 
Our technical contributions are summarized as follows:
\begin{itemize}
    \item We develop new algorithms for computing both complex and real-valued rank-$1$ Hankel matrix approximations of arbitrary matrices (complex or real) under both $L_2$ and $L_1$-norm error formulations. 
    
    \item As an application of interest, we develop new algorithms for DoA estimation with limited data (few-shot) based on rank-$1$ Hankel matrix approximations under the $L_2$ and $L_1$-norm formulations. Under the $L_2$-norm, we prove that the resulting estimator is maximum-likelihood optimal 
    in the presence of white Gaussian noise.
     Under the $L_1$-norm, we prove that the derived estimator is maximum-likelihood optimal in the presence of Laplace noise. 

    \item We conduct extensive simulations with varying antenna array sizes and real-world data experiments to evaluate the developed Hankel-based few-shot DoA estimators. We demonstrate that the $L_2$-norm estimator exhibits dominant state-of-the-art performance in white Gaussian noise sensing environments. By leveraging synthetic impulsive-noise scenarios and UAV (unmanned aerial vehicle) collected measurements, we examine challenging environments with array imperfections. We demonstrate that in such environments the developed $L_1$-norm estimator exhibits outstanding robustness and accuracy compared to competing frameworks.  
    
\end{itemize}

The rest of the paper is organized as follows. Section II introduces the problem statement, provides technical preliminaries, and presents the notation used throughout the paper. Section III describes in detail the proposed algorithms for the computation of $L_2$ and $L_1$-norm rank-$1$ Hankel-structured approximations, while Section IV demonstrates their application to DoA estimation. Section V presents simulation studies and real-world data experimental results. A few conclusions are drawn in Section VI.

\textit{Notation}: In this paper, matrices are denoted by upper-case bold letters, column vectors by lower-case bold letters, and scalars by lower-case plain-font letters. The transpose operation is represented by the superscript $(\cdot)^T$, conjugation by $(\cdot)^*$, the conjugate transpose (Hermitian) by $(\cdot)^H$, and the Kronecker product by $\otimes$. 

\section{Problem Statement and Notation}

Hankel and Toeplitz matrices are two fundamental structured matrix classes that have attracted significant attention in engineering applications over the years. 
An arbitrary matrix $\mathbf{H} \in \mathbb{C}^{D\times W}$ possesses a Hankel structure if each anti-diagonal has elements of constant value, or equivalently, its entry values only depend on the sum of their indices minus one. That is, a Hankel matrix has elements of the form
$\mathbf{H}_{i,j} = \mathbf{h}_{i+j-1}, {i=1, \cdots, D}, \, {j=1,\cdots,W},$ and
\begin{equation}
\label{hankel_matrix}
\mathbf{H}_{D\times W} =  \left[\begin{array}{ccccc}
h_1 & h_2 &  \cdots & h_W \\
h_2 & h_3 &  \cdots & h_{W+1} \\
h_3 & h_4 &  \cdots & h_{W+2} \\
\vdots &  \vdots & \ddots & \vdots \\
h_D & h_{D+1} &  \cdots & h_{D+W-1}
\end{array}\right]
\end{equation} 
for some parameter vector
$\mathbf{h} \in \mathbb{C}^{M}$,  $M=D+W-1$, that constitutes the entries of the first column and last row,
\begin{equation}
\mathbf{h} \triangleq \left[h_1, h_{2}\ldots, h_{D}, h_{D+1}, \ldots h_{D+W-1}\right]^T.
\end{equation}
We can then define the Hankel operator $\mathcal{H}: \mathbb{C}^M \rightarrow \mathbb{C}^{D \times W}$ that maps a parameter vector to its corresponding Hankel representation by placing side-by-side $W = M-D+1$ sub-vectors  of $\mathbf{h}$ of length $D$ generated by successive one-entry shifts. 

An arbitrary matrix has a Toeplitz structure if it exhibits constant values along its main diagonals. 
Hankel and Toeplitz matrices are closely related, as any Toeplitz matrix admits an equivalent Hankel representation obtained through index reversal. Specifically, defining the $D \times D$ exchange matrix 
\begin{equation}
\label{reversal_matrix}
\mathbf{J}_D \triangleq
\begin{bmatrix}
0 & 0 &  \cdots & 0 & 1 \\
0 & 0 &  \cdots & 1 & 0 \\
\vdots & \vdots  & \iddots & \vdots & \vdots \\
0 & 1 &  \cdots & 0 & 0 \\
1 & 0 &  \cdots & 0 & 0
\end{bmatrix}
\end{equation}
where $\mathbf{J}_D = \mathbf{J}_D^T = \mathbf{J}_D^{-1}$,
any Toeplitz matrix $\mathbf{T} \in \mathbb{C}^{D\times W}$ can be represented by a corresponding Hankel matrix $\mathbf{H} \in \mathbb{C}^{D\times W}$ by 
\begin{equation}
\label{T_H}
\mathbf{T}=\mathbf{J}_D \mathbf{H}.
\end{equation} 
Considering that {\em{(i)}} element-wise $L_2$ matrix norm (i.e., Frobenius norm) $\|\mathbf{A}\|_2 \triangleq \sqrt{\sum_{i,j} \left|a_{i,j}\right|^2}$ is unitarily invariant, {\em{(ii)}} element-wise $L_1$ matrix norm $\|\mathbf{A}\|_1 \triangleq {\sum_{i,j} \left|a_{i,j}\right|}$ is invariant under permutation matrices such as the exchange matrix $\mathbf{J}$, and {\em{(iii)}}  matrix rank is preserved under multiplication by an invertible matrix, it follows from (\ref{T_H}) that the Toeplitz approximation problem of an arbitrary matrix $\mathbf{X}\in \mathbb{C}^{D \times W}$ is  equivalent to the Hankel approximation problem of $\mathbf{J}_D \mathbf{X}$ under either norm, 
\begin{equation}
\begin{split}
\min _{\mathbf{T}}\|\mathbf{X}-\mathbf{T}\| =\min _{\mathbf{H}}\left\|\mathbf{X}-\mathbf{J}_D \mathbf{H}\right\|= \\
\min _{\mathbf{H}}\left\|\mathbf{J}_D\left(\mathbf{X}-\mathbf{J}_D \mathbf{H}\right)\right\|=\min _{\mathbf{H}}\left\|\mathbf{J}_D \mathbf{X}-\mathbf{H}\right\| .
\end{split}
\end{equation} 

In the following, for clarity and without loss of generality, we consider only Hankel matrix approximations. In particular, we address the problem of approximating an arbitrary matrix $\mathbf{X}\in \mathbb{C}^{D \times W}$ by a rank-$1$ Hankel matrix. Rank-$1$ Hankel approximations give rise to the challenging non-convex optimization problem
\begin{equation}
\label{og_rank1_formulation}
  \mathbf{H}_{opt}=\underset{\substack{ \mathbf{H} \in \mathbb{H}, \,\,
  \operatorname{rank}(\mathbf{H})=1}}{\operatorname{argmin}}\|\mathbf{X}-\mathbf{H}\|
\end{equation}
where $\mathbb{H}$ denotes the set of Hankel matrices of dimension $D \times W$ and $\|\cdot\|$ represents the matrix norm of interest ($L_2$ or $L_1$, herein). We define the normalized polynomial structure vector in $z \in \mathbb{C}$
\begin{equation}
\label{structure_vector}
\mathbf{s}_D(z) \triangleq \frac{\hat{\mathbf{s}}_D(z)}{\left\|\hat{\mathbf{s}}_D(z)\right\|_2}
\end{equation} where $\hat{\mathbf{s}}_D(z) \triangleq\left[1, z, z^2, \ldots, z^{D-1}\right]^{T}$ and
\begin{equation}
\label{norm}
\left\|\hat{\mathbf{s}}_D(z)\right\|_2 = \begin{cases}\sqrt{\frac{1-|z|^{2 D}}{1-|z|^2}}, & |z| \neq 1, \\ \sqrt{D}, & |z|=1 .\end{cases}
\end{equation}
It is known \cite{knirsch2021optimal} that  a matrix $\mathbf{H}\in \mathbb{C}^{D \times W}$ is Hankel rank-1 if and only if it can be written as 
\begin{equation}
\label{rank1_lemma}
    \mathbf{H} = c \mathbf{s}_D(z) \mathbf{s}_W(z)^{T}
\end{equation} for some $c \in \mathbb{C} \backslash\{0\}$ and $z \in \mathbb{C} \cup\{\infty\}$. It directly follows from (\ref{rank1_lemma}) that the optimization problem in (\ref{og_rank1_formulation}) can be analytically reformulated in terms of two complex scalar parameters \cite{knirsch2021optimal} as
\begin{equation}
\label{mod_rank_1_formulation}
(\hat{c}, \hat{z})=\underset{c \in \mathbb{C} \backslash\{0\}, \text{ } z \in \mathbb{C} \cup\{\infty\}}{\operatorname{argmin}}\left\|\mathbf{X}-c \mathbf{s}_D(z) \mathbf{s}_W(z)^T\right\| 
\end{equation}
with $\mathbf{H}_{opt} = \hat{c}\mathbf{s}_D(\hat{z}) \mathbf{s}_W(\hat{z})^T$.
To avoid the inconvenient case of $z = \infty$, it suffices to assume that $\left|x_{1,1}\right| \geq\left|x_{D, W}\right|$. This assumption does not pose any practical restriction on $\mathbf{X}$, as any matrix that violates it can be replaced by its flipped counterpart $\mathbf{J}_D \mathbf{X}\mathbf{J}_W$ \cite{knirsch2021optimal}.
In the following, we are interested in developing algorithms that directly solve the optimization problem in (\ref{mod_rank_1_formulation}) under both the $L_2$ and $L_1$-norm. 

\section{Algorithms for $L_2$-norm and $L_1$-norm Rank-$1$ Hankel Approximations}
\subsection{$L_2$-norm Rank-$1$ Hankel Matrix Approximations}
\label{l2_section}
We revisit the approximation problem in (\ref{mod_rank_1_formulation}) under the $L_2$-norm
\begin{equation}
\label{l2_rank1}
(\hat{c}_{L_2}, \hat{z}_{L_2})=\underset{c \in \mathbb{C} \backslash\{0\},  z \in \mathbb{C} }{\operatorname{argmin}}\left\|\mathbf{X}-c \mathbf{s}_D(z) \mathbf{s}_W(z)^T\right\|_2
\end{equation}
to produce an optimal rank-1 Hankel approximation of $\mathbf{X}$ of the form
$\mathbf{H}_{opt}^{L_2} = \hat{c}_{L_2} \mathbf{s}_D(\hat{z}_{L_2}) \mathbf{s}_W(\hat{z}_{L_2})^T$.
Direct differentiation of (\ref{l2_rank1}) with respect to $c$ collapses the search to a single complex variable $z$,
\begin{align}
\label{theorem_1}
    \hat{z}_{L_2} & = \underset{z \in \mathbb{C}}{\operatorname{argmax}}\left|\mathbf{s}_D(z)^H \mathbf{X} \mathbf{s}_W(z)^*\right|, \\
    \hat{c}_{L_2} & =\mathbf{s}_D(\hat{z}_{L_2})^H \mathbf{X} \mathbf{s}_W(\hat{z}_{L_2})^*.
    \label{c}
\end{align}
To obtain a practical numerical algorithm, we exploit Theorem 3.4 in \cite{knirsch2021optimal} which restricts the search domain from the full complex plane to the complex unit disc through two complementary optimization problems, 
\begin{equation}
A = \underset{z \in \mathbb{C}, |z|\leq1}{\operatorname{max}}\left|\mathbf{s}_D(z)^H \mathbf{X} \mathbf{s}_W(z)^*\right|,
\label{F2}
\end{equation}
\begin{equation}
B= \underset{z \in \mathbb{C}, |z|\leq1}{\operatorname{max}}\left|\mathbf{s}_D(z)^H \mathbf{J}_D \mathbf{X} \mathbf{J}_W \mathbf{s}_W(z)^*\right|
\label{F2_prime}
\end{equation}
where $\mathbf{J}_D$, $\mathbf{J}_W$ are the exchange matrices as in (\ref{reversal_matrix}). 
If $A \geq B$, then 
\begin{equation}
   \hat{z}_{L_2} =  \underset{z \in \mathbb{C}, |z|\leq1}{\operatorname{argmax}}\left|\mathbf{s}_D(z)^H \mathbf{X} \mathbf{s}_W(z)^*\right|.
   \label{one}
\end{equation}
Else,
\begin{equation}
   \hat{z}_{L_2} = \left[  \underset{z \in \mathbb{C}, |z|\leq1}{\operatorname{argmax}}\left|\mathbf{s}_D(z)^H \mathbf{J}_D \mathbf{X} \mathbf{J}_W\mathbf{s}_W(z)^*\right| \right]^{-1}.
   \label{two}
\end{equation}

To solve (\ref{F2}), (\ref{F2_prime}), we carry out polar grid search over  $z = \rho e^{j\phi}$, $\rho \leq 1$, with granularity step $\Delta \rho$ and $\Delta \phi$, respectively. The algorithmic complexity becomes  $\mathcal{O}( \frac{1}{\Delta \rho}  \frac{2 \pi}{\Delta \phi}  DW)$. The algorithm is summarized in Fig. \ref{l2_grid_alg} for easy reference.

\begin{figure}[htbp]
\renewcommand{\arraystretch}{1.15}
{\fontsize{9pt}{9}\selectfont
    {\hrule height 0.2mm}
    \vspace{0.3mm}
    {\hrule height 0.2mm}
    \vspace{1mm}
    {\bf Algorithm: Complex Rank-$1$ Hankel Approximation in $L_2$-norm}
    \vspace{0.0mm}
    {\hrule height 0.2mm}
    \vspace{2mm}
    \textbf{Input: } Data matrix $\mathbf{X} \in \mathbb{C}^{D \times W}$, radial step $\Delta \rho$, angular step $\Delta \phi$\\[1mm]
    \begin{tabular}{r l}
        1: & Generate grid $\mathcal{G}_z$ to solve (\ref{F2}), (\ref{F2_prime}). \\[1mm]
        2: & Compute $ A = \underset{z \in \mathcal{G}_z}{\operatorname{max}}\left|\mathbf{s}_D(z)^H \mathbf{X} \mathbf{s}_W(z)^*\right|$ in (\ref{F2}). \\[1mm]
        3: & Compute $ B = \underset{z \in \mathcal{G}_z}{\operatorname{max}}\left|\mathbf{s}_D(z)^H  \mathbf{J}_D \mathbf{X}  \mathbf{J}_W \mathbf{s}_W(z)^*\right|$ in (\ref{F2_prime}). \\[1mm]
        4: & \textbf{if} $A \geq B$ \textbf{then}: \\[1mm]
        5: & \hspace{12pt} $\hat{z}_{L_2} = \underset{z \in \mathcal{G}_z}{\operatorname{argmax}}\left|\mathbf{s}_D(z)^H \mathbf{X} \mathbf{s}_W(z)^*\right|$ \\[1mm]
        6: & \textbf{else}: \\[1mm]
        7: & \hspace{12pt} $\hat{z}_{L_2} = \left[\underset{z \in \mathcal{G}_z}{\operatorname{argmax}} \left|\mathbf{s}_D(z)^H  \mathbf{J}_D \mathbf{X}  \mathbf{J}_W \mathbf{s}_W(z)^*\right| \right]^{-1}$. \\[1mm]
        8: & \textbf{end if} \\[1mm]
        9: &  Calculate $\hat{c}_{L_2}=\mathbf{s}_D(\hat{z}_{L_2})^H \mathbf{X} \mathbf{s}_W(\hat{z}_{L_2})^*$ in (\ref{c}).\\[1mm]
        10: & Form $\mathbf{H}_{opt}^{L_2} = \hat{c}_{L_2}\,\mathbf{s}_D(\hat{z}_{L_2})\,\mathbf{s}_W(\hat{z}_{L_2})^{T}$ in (\ref{rank1_lemma}).\\[1mm]
    \end{tabular} \\[1mm]
    \textbf{Output: } Rank-1 Hankel approximation $\mathbf{H}_{opt}^{L_2} \in \mathbb{C}^{D \times W}$. \\[1mm]
    {\hrule height 0.2mm}
    \vspace{0.3mm}
    {\hrule height 0.2mm}
    \vspace{0.4cm}
}
\caption{Algorithm for complex rank-1 Hankel approximation of arbitrary complex matrices in $L_2$-norm.}
\vspace{-0.2cm}
\label{l2_grid_alg}
\renewcommand{\arraystretch}{1}
\end{figure}

It is worth noting that the algorithm as presented in Fig. \ref{l2_grid_alg} can be readily applied to produce real-valued rank-$1$ Hankel matrix approximations $\mathbf{H}_{opt}^{L_2} \in \mathbb{R}^{D \times W}$ of arbitrary complex or real matrices ${\mathbf{X}}_{D \times W}$ by restricting the search for $z = \rho e^{j\phi}$, $\rho \leq 1$, to $\phi = 0$. Then, $z \in [-1, 1]$ and grid search complexity reduces to 
 $\mathcal{O}( \frac{1}{\Delta \rho}   DW)$. It is also worth noting that, as is, the algorithm in  Fig. \ref{l2_grid_alg} can produce complex approximations of real-valued input matrices supporting quadrature signal representation applications.

\subsection{$L_1$-norm Rank-$1$ Hankel Matrix Approximations}
Next, we consider rank-1 Hankel approximations of arbitrary matrices under  
 the $L_1$-norm, 
\begin{equation}
\label{l1_rank1}
(\hat{c}_{L_1}, \hat{z}_{L_1})=\underset{c \in \mathbb{C} \backslash\{0\}, \text{ } z \in \mathbb{C} }{\operatorname{argmin}}\left\|\mathbf{X}-c \mathbf{s}_D(z) \mathbf{s}_W(z)^T\right\|_1.
\end{equation}
For clarity in presentation, we introduce the normalization factor
\begin{equation}
\label{alpha_z}
\alpha(z) \triangleq
\frac{1}{\sqrt{\sum_{i=1}^{D} |z|^{2(i-1)}}\sqrt{\sum_{j=1}^{W} |z|^{2(j-1)}}}
\end{equation}
and then proceed to expand (\ref{l1_rank1}) to
\begin{IEEEeqnarray}{rCl}
\IEEEeqnarraymulticol{3}{l}{
\underset{c \in \mathbb{C} \backslash\{0\}, z \in \mathbb{C}}{\operatorname{argmin}}\left\|\mathbf{X}-c \mathbf{s}_D(z) \mathbf{s}_W(z)^{\top}\right\|_1 =
} \nonumber\\[4pt]
&=& \displaystyle \underset{c \in \mathbb{C} \backslash\{0\}, z \in \mathbb{C}}{\operatorname{argmin}} \sum_{i=1}^D \sum_{j=1}^W|z|^{i+j-2}\left|x_{i j} z^{-(i+j-2)}-c \alpha(z)\right|. \nonumber\\[3pt]
\label{l1_expand}
\end{IEEEeqnarray}
For any fixed $z \in \mathbb{C}$ and $\tilde{c} \triangleq c \alpha(z)$, the minimization over $c$ is equivalent to minimization over $\tilde{c}$, 
\begin{equation}
\label{geometric_median_problem}
\hat{\tilde{c}}(z)=\underset{\tilde{c} \in \mathbb{C}}{\operatorname{argmin}} \sum_{i=1}^D \sum_{j=1}^W|z|^{i+j-2}\left|x_{i j} z^{-(i+j-2)}-\tilde{c}\right|.
\end{equation} Then, the optimal scalar coefficient $\hat{c}$ for the original optimization problem can be recovered as 
\begin{equation}
\label{c_l1}
   \hat{c}_{L_1}(z) = \frac{\hat{\tilde{c}}(z)}{\alpha(z)}.
\end{equation}

We recognize that the optimization problem in (\ref{geometric_median_problem}) corresponds to computing the weighted geometric median, which can be efficiently calculated via Weiszfeld's iterative algorithm \cite{weiszfeld1937point}. 
Therefore, unlike the $L_2$-norm case where the optimal coefficient $c$ admits a closed-form analytical solution, the $L_1$-norm formulation yields an iterative computation procedure. Reduced to a single variable $z \in \mathbb{C}$, the optimization problem of interest is 
\begin{equation}
\hat{z}_{L_1}=\underset{z \in \mathbb{C}}{\operatorname{argmin}}\left\|\mathbf{X}-\hat{c}_{L_1}(z) \mathbf{s}_D(z) \mathbf{s}_W(z)^T\right\|_1,
\end{equation}
with $\hat{c}_{L_1}(z)$ in (\ref{c_l1}).

To restrict the grid search to the unit disc, $z = \rho e^{j\phi}$, $\rho \leq 1$, we solve two minimization problems,
\begin{equation}
C = \underset{z \in \mathbb{C}, |z|\leq 1}{\operatorname{min}}\left\|\mathbf{X}-\hat{c}_{L_1}(z) \mathbf{s}_D(z) \mathbf{s}_W(z)^T\right\|_1,
\label{equationC}
\end{equation}
\begin{equation}
 D =   \underset{z \in \mathbb{C}, |z|\leq 1}{\operatorname{min}}\left\|\mathbf{X}-\hat{c}_{L_1}(1/z) \mathbf{s}_D(1/z) \mathbf{s}_W(1/z)^T\right\|_1.
 \label{equationD}
\end{equation}
If $C \leq D$, 
\begin{equation}
\hat{z}_{L_1} =    \underset{z \in \mathbb{C}, |z|\leq 1}{\operatorname{argmin}}\left\|\mathbf{X}-\hat{c}_{L_1}(z) \mathbf{s}_D(z) \mathbf{s}_W(z)^T\right\|_1.
\label{F1}
\end{equation}
Else, 
\begin{equation}
    \hat{z}_{L_1} =  \left[ \underset{z \in \mathbb{C}, |z|\leq 1}{\operatorname{argmin}}\left\|\mathbf{X}-\hat{c}_{L_1}(1/z) \mathbf{s}_D(1/z) \mathbf{s}_W(1/z)^T\right\|_1 \right]^{-1}.
    \label{F1_prime}
\end{equation}

To solve (\ref{F1}), (\ref{F1_prime}), we carry out polar grid search over  $z = \rho e^{j\phi}$, $\rho \leq 1$, with granularity steps $\Delta \rho$ and $\Delta \phi$. In contrast to the $L_2$-norm case, there is no closed from expression available for the scaling parameter $c$ in (\ref{geometric_median_problem}), (\ref{c_l1}). As a result, while still manageable, the algorithmic complexity increases to  $\mathcal{O}( \frac{1}{\Delta \rho}  \frac{2 \pi}{\Delta \phi}  D^2W^2 T)$ where $T$ is the iteration bound in the Weiszfeld procedure \cite{weiszfeld1937point}. The algorithm is summarized in Fig. \ref{l1_grid_alg} for easy reference.

 
\begin{figure}[htbp]
\renewcommand{\arraystretch}{1.15}
{\fontsize{9pt}{9}\selectfont
    {\hrule height 0.2mm}
    \vspace{0.3mm}
    {\hrule height 0.2mm}
    \vspace{1mm}
    {\bf Algorithm: Complex Rank-$1$ Hankel Approximation in $L_1$-norm}
    \vspace{0.0mm}
    {\hrule height 0.2mm}
    \vspace{2mm}
    \textbf{Input: } Data Matrix $\mathbf{X} \in \mathbb{C}^{D \times W}$, radial step $\Delta \rho$, angular step $\Delta \phi$\\[1mm]
    \begin{tabular}{r l}
        1: & Generate grid $\mathcal{G}_z$ to solve (\ref{equationC}), (\ref{equationD}). \\[1mm]
        2: & Compute $C = \underset{z \in \mathcal{G}_z}{\operatorname{min}}\left\|\mathbf{X}-\hat{c}_{L_1}(z) \mathbf{s}_D(z) \mathbf{s}_W(z)^T\right\|_1$ in (\ref{equationC}).
       \\[1mm]
        3: & Compute $D =   \underset{z \in \mathcal{G}_z}{\operatorname{min}}\left\|\mathbf{X}-\hat{c}_{L_1}(1/z) \mathbf{s}_D(1/z) \mathbf{s}_W(1/z)^T\right\|_1$\\ 
        & in (\ref{equationD}).
        \\[1mm]
        4: & \textbf{if} $C \leq D$ \textbf{then}: \\[1mm]
        5: & \hspace{12pt} $\hat{z}_{L_1} =    \underset{z \in \mathcal{G}_z}{\operatorname{argmin}}\left\|\mathbf{X}-\hat{c}_{L_1}(z) \mathbf{s}_D(z) \mathbf{s}_W(z)^T\right\|_1$. \\ [1mm]
        6: & \textbf{else}: \\[1mm]
        7: & \hspace{12pt} 
         $\hat{z}_{L_1} = \left[ \underset{z \in \mathcal{G}_z}{\operatorname{argmin}}\left\|\mathbf{X}-\hat{c}_{L_1}(1/z) \mathbf{s}_D(1/z) \mathbf{s}_W(1/z)^T\right\|_1 \right]^{-1}.$ \\[1mm]
        8: & \textbf{end if} \\[1mm]
        9: & Form $\mathbf{H}_{opt}^{L_1} = \hat{c}_{L_1}\,\mathbf{s}_D(\hat{z}_{L_1})\,\mathbf{s}_W(\hat{z}_{L_1})^{T}$ in (\ref{rank1_lemma}).\\[1mm]
    \end{tabular} \\[1mm]
    \textbf{Output: } Rank-1 Hankel approximation $\mathbf{H}_{opt}^{L_1} \in \mathbb{C}^{D \times W}$. \\[1mm]
    {\hrule height 0.2mm}
    \vspace{0.3mm}
    {\hrule height 0.2mm}
    \vspace{0.4cm}
}
\caption{Algorithm for complex rank-1 Hankel approximation of arbitrary complex matrices in $L_1$-norm.}
\vspace{-0.2cm}
\label{l1_grid_alg}
\renewcommand{\arraystretch}{1}
\end{figure}

The algorithm in Fig. \ref{l1_grid_alg} can be used as is to produce complex rank-1 Hankel approximations of real data matrices $\mathbf{X} \in \mathbb{R}^{D \times W}$. The algorithm can produce real approximations of real or complex data matrices by restricting the grid search for $z$ to $z \in [-1, 1]$. Then, the algorithmic complexity reduces to $\mathcal{O}( \frac{1}{\Delta \rho}  D^2W^2 T)$.

\section{Application to Signal Direction-of-Arrival Estimation}
The problem of estimating the Direction-of-Arrival (DoA) of a signal has been at the forefront of array signal processing for decades \cite{van2002optimum,526899}, with pivotal technological applications in diverse areas including but not limited to wireless communications, remote sensing, radar and sonar technologies, and nowadays autonomous systems and robotic mobility \cite{622504,7815340,9219157,9725025,9722876, 9148550,9726790,10160524,10018231}. Unstructured Low-Rank  Approximation (ULRA) of the received signal autocorrelation matrix gave rise to highly successful subspace-based signal direction finding methods, such as MUSIC \cite{1143830}, ESPRIT \cite{32276}, and other \cite{80966,4063549,9516894}.
With the advent of highly mobile autonomous platforms operating in environments of severely limited statistical coherence time, 
the ability to wait and collect statistically stationary antenna-array snapshots to form a high-quality estimate of the input autocorrelation matrix is limited - if not completely negated. As a result, we need to consider scenaria where DoA estimation must be performed from only few antenna-array snapshots \cite{10051870,10096647}. 
Moreover, while the number of available sensing  elements has increased significantly in deployed large and massive MIMO antenna array systems, there are stringent cost-induced constraints in the number of independent RF chains and analog-to-digital converters (ADC) that operate simultaneously in the backend to produce independent data samples.

Few-shot signal DoA estimation on large arrays with limited number of RF ADC
chains constitutes a direct application of interest to the theory and methods of Hankel-structured decompositions of arbitrary complex data matrices, as developed in Section III. Details follow below.

\subsection{Signal Model and Sensing Architecture}
\label{signal_model}
For clarity and simplicity in our presentation, we consider a uniform linear antenna array (ULA) with $M$ elements. If $\theta \in\left[-90^{\circ}, 90^{\circ}\right)$ denotes the signal incidence angle with respect to the broadside, then the array response vector is given by
\begin{equation}
\label{array_response_vector}
    \mathbf{a}_M(\theta)\triangleq\left[1, e^{-j 2 \pi \frac{d}{\lambda} \sin \theta}, \cdots, e^{-j 2 \pi(M-1) \frac{d}{\lambda} \sin \theta}\right]^T
\end{equation}
where $d$ is the inter-element spacing and $\lambda$ is the carrier wavelength. Considering the fundamental problem where a single narrowband signal from the far field impinges on the antenna array, the baseband received antenna-array sample, $\mathbf{r} \in \mathbb{C}^M$, is of the form 
\begin{equation}\label{signal_model_1_signal}
    \mathbf{r} = x \mathbf{a}_M(\theta) + \mathbf{n}
\end{equation}
where $x \in \mathbb{C}$ denotes the complex fixed unknown signal amplitude and $\mathbf{n} \in \mathbb{C}^M$ represents additive noise.

\begin{figure}[htbp]
    \centering
    \includegraphics[width=0.8
    \linewidth]{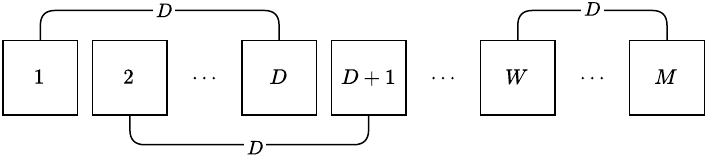}
    \caption{Sliding-window acquisition over a uniform linear array (ULA) of $M$ antenna elements using $W$ overlapping subarrays of length $D$.}
    \label{fig:sliding_array_arch}
\end{figure}

We consider a sliding subarray data acquisition model, illustrated in Fig. \ref{fig:sliding_array_arch}, in which only a subset of $D$ antenna elements is accessible at a given time by available processing RF chains. By sequentially sliding a window of length $D$ across the array, the full aperture is synthesized using $W=M-D+1$ overlapping subarrays. Specifically, the $i$th subarray data vector is given by 
\begin{equation}
\label{sliding_sub}
\begin{split}
    \mathbf{r}_i=x\left[a_M(\theta)_{i+1}, a_M(\theta)_{i+2}, \ldots, a_M(\theta)_{i+D} \right]^T + \mathbf{n}_i, \\ \quad i = 0, \ldots, W-1,
\end{split}
\end{equation}
where $\mathbf{n}_i \in \mathbb{C}^D$ denotes the additive noise vector associated with the $i$th subarray acquisition. 
%

\subsection{Rank-$1$-Hankel Direction-of-Arrival Estimation}
Next, we stack the subarray acquired observations side-by-side to form the combined measurement matrix $\mathbf{X} \in \mathbb{C}^{D \times W}$, 
\begin{equation}
\label{agg_meas_matrix}
    \mathbf{X} = [\mathbf{r}_0,\cdots,\mathbf{r}_{W-1}],
\end{equation} where the $i$th column of $\mathbf{X}$ corresponds to the observation obtained from the $i$th sliding subarray $\mathbf{r}_i$ in (\ref{sliding_sub}). 
While $\mathbf{X}$ does not, in general, possess a Hankel structure due to the individual noise contributions across subarrays, the underlying signal component admits an exact rank-$1$ Hankel representation
\begin{equation}
    \mathbf{X}_{S} = x 
    \left[\begin{array}{c}
    1 \\
    e^{-j 2 \pi \frac{d}{\lambda} \sin \theta} \\
    
    \vdots \\
    e^{-j 2 \pi (D-1) \frac{d}{\lambda} \sin \theta}
    \end{array}\right]
    \left[\begin{array}{c}
    1 \\
    e^{-j 2 \pi \frac{d}{\lambda} \sin \theta} \\
    \vdots \\
    e^{-j 2 \pi (W-1) \frac{d}{\lambda} \sin \theta}
    \end{array}\right]^T,
\end{equation}
which is of the form 
\begin{equation}
   \mathbf{X}_{S} =  c \mathbf{s}_D(z) \mathbf{s}_W(z)^T
\end{equation}
with 
\begin{equation}
    z = e^{-j 2 \pi \frac{d}{\lambda} \sin \theta}, 
    \quad \theta \in [-90^\circ,90^\circ).
\end{equation} 

Consequently,  in parallel to the technical developments in Section III we propose the following two signal Direction-of-Arrival estimators:
\begin{equation}
\label{l2_rank1_doa}
\hat{\theta}_{L_2}=\underset{\theta \in\left[-90^{\circ}, 90^{\circ}\right)}{\operatorname{argmin}}\left\|\mathbf{X}-\hat{c}_{L_2}(z(\theta)) \mathbf{s}_D(z(\theta)) \mathbf{s}_W(z(\theta))^T\right\|_2
\end{equation} 
and 
\begin{equation}
\label{l1_rank1_doa}
\hat{\theta}_{L_1} =\underset{\theta \in\left[-90^{\circ}, 90^{\circ}\right)}{\operatorname{argmin}}\left\|\mathbf{X}-\hat{c}_{L_1}(z(\theta)) \mathbf{s}_D(z(\theta)) \mathbf{s}_W(z(\theta))^T\right\|_1.
\end{equation} 
The estimate in (\ref{l2_rank1_doa}) is computed by the algorithm in Fig. \ref{l2_grid_alg} and the estimate in (\ref{l1_rank1_doa}) is computed by the algorithm in Fig. \ref{l1_grid_alg}, both with a simplified search over $z \in \mathbb{C}$ with $|z|=1$ and phase $\theta \in [-90^\circ,90^\circ)$. The following two Theorems establish the optimality of $\hat{\theta}_{L_2}$ and $\hat{\theta}_{L_1}$ under corresponding statistical assumptions. The proofs are given in the Appendix.

\begin{theorem}
\label{ml_gaussian}
Under independent identically distributed (i.i.d.)\ isotropic complex Gaussian noise per sensor, the proposed estimator in (\ref{l2_rank1_doa}) is maximum-likelihood optimal. \hfill $\blacksquare$
\end{theorem}

\begin{theorem}
\label{ml_laplace}
Under i.i.d.\ isotropic complex Laplace noise per sensor, the proposed estimator in (\ref{l1_rank1_doa}) is maximum-likelihood optimal. \hfill $\blacksquare$
\end{theorem}

\section{Simulations and Experiments}

In this section, we study experimentally several direction-of-arrival (DoA) estimators under the sensing model described in Section \ref{signal_model}, on both simulated data and real-world UAV measurements. In particular, we carry out comparative performance evaluation of the proposed rank-$1$-Hankel decomposition DoA estimators under the $L_2$-norm and $L_1$-norm and the following leading DoA estimators from the literature: \emph{(i)} The Matrix Pencil method, originally developed for pole estimation in linear systems \cite{hua2002matrix} and later adapted for DoA estimation \cite{adve1996elimination} through Hankel data matrix construction; \emph{(ii)} Hankel-MUSIC \cite{liao2016music} where the noise subspace is drawn from the Singular Value Decomposition of Hankel-matrix data; and \emph{(iii)} Forward-Backward Spatial Smoothing (FB-SS) MUSIC \cite{1164649,9564893}, which estimates Hankel-structured spatially smoothed covariance matrices by averaging subarray covariances. Direct data averaging per sensor to create a single data point $\mathbf{\tilde{r}} \in \mathbb{C}^M$ gives rise to two additional competitive methods: \emph{(iv)} Maximum energy estimation, $\hat{\theta} = \underset{\theta }{\operatorname{argmax}}\left|\mathbf{a}_M(\theta)^H \mathbf{\tilde{r}}\right|$; and \emph{(v)} MUSIC on Toeplitz covariance matrix estimate from $\mathbf{\tilde{r}}$  \cite{degen2017single}, \cite{marple2019digital}.

In all simulations, we follow the signal model in (\ref{signal_model_1_signal}). We consider uniform linear antenna arrays (ULA) with Nyquist inter-element spacing $d = \frac{\lambda}{2}$.
We examine ULAs with $M=4,8,16,32,$ and $128$ elements and $D=M/2$ independent RF processing chains utilized as in Fig. \ref{fig:sliding_array_arch}.
We consider two different noise conditions: First, standard i.i.d. white Gaussian noise across time and element and, second, i.i.d. Gaussian mixture impulsive noise across time and elements. 

We begin with the white Gaussian noise case under which  the noise vector in  (\ref{signal_model_1_signal})  is distributed as
\begin{equation}
\label{awgn}
\mathbf{n} \sim \mathcal{CN}\left(\mathbf{0},\, \sigma^2 \mathbf{I}_{D}\right) 
\end{equation}
where $\sigma^2$ denotes the noise variance per antenna element and $\mathbf{I}_D$ is the $D \times D$ identity matrix. 
For fixed unknown signal amplitude $x \in \mathbb{C}$ across samples
$\mathbf{r}_i$, $i = 0, \ldots, W-1$, in (\ref{sliding_sub}), the signal-to-noise ratio defined as $\text{SNR} \triangleq \frac{|x|^2}{\sigma^2}$dB is controlled by setting $x = \sqrt{10^{SNR / 10} \sigma^2} e^{j \phi}$ where $\phi$ is drawn from the uniform $[0, 2\pi]$ distribution.
We recall that under the assumption of white Gaussian noise the proposed rank-$1$ Hankel-decomposition DoA estimator in Fig. \ref{l2_grid_alg} under the $L_2$-norm is maximum-likelihood optimal. 


Fig. \ref{fig:mae_all_wgn} reports the mean absolute estimation error, averaged over $5000$ independent experiments, of the proposed rank-$1$ Hankel decomposition method under the $L_2$-norm and all other estimators \emph{(i)-(v)} described above, for SNR levels  $-5$dB, $0$dB, $5$dB, and $10$dB. As expected, performance improves monotonically with increasing SNR and number of antenna elements for all methods. Across all configurations considered, the proposed DoA estimator outperforms the competing approaches and the performance gap opens as the number of antennas and/or the SNR increases. Notably, at SNR of $0$dB or better and $M=64$ antennas the average absolute estimation error of the proposed algorithm is less that one-tenth of a degree. At $10$dB SNR and $M=128$, the absolute estimation error is as low as one-hundredth of a degree.

\begin{figure*}[htbp]
    \centering
    \begin{minipage}[t]{0.49\textwidth}
        \centering
        \subfloat[]{
            \includegraphics[width=\textwidth]{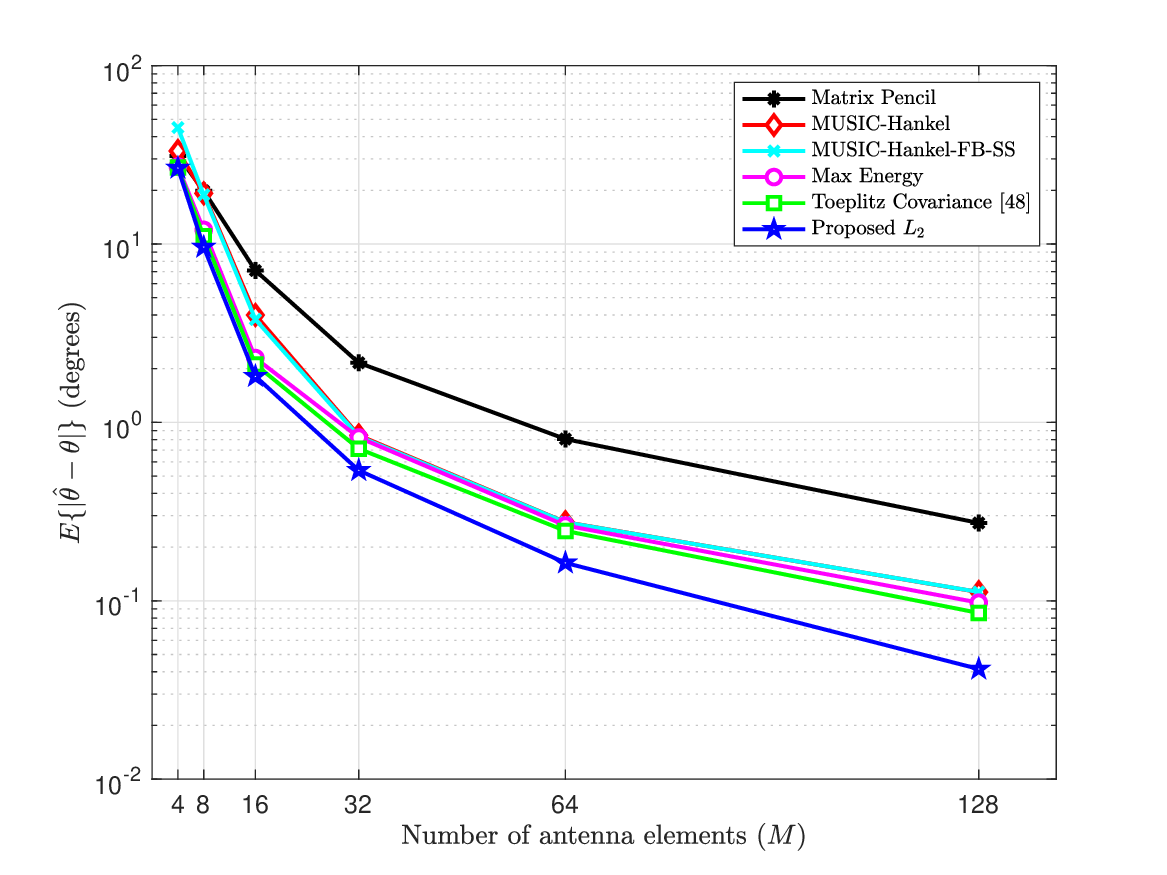}
            \label{fig:mae_snr_-5}
        } \\
        \subfloat[]{
            \includegraphics[width=\textwidth]{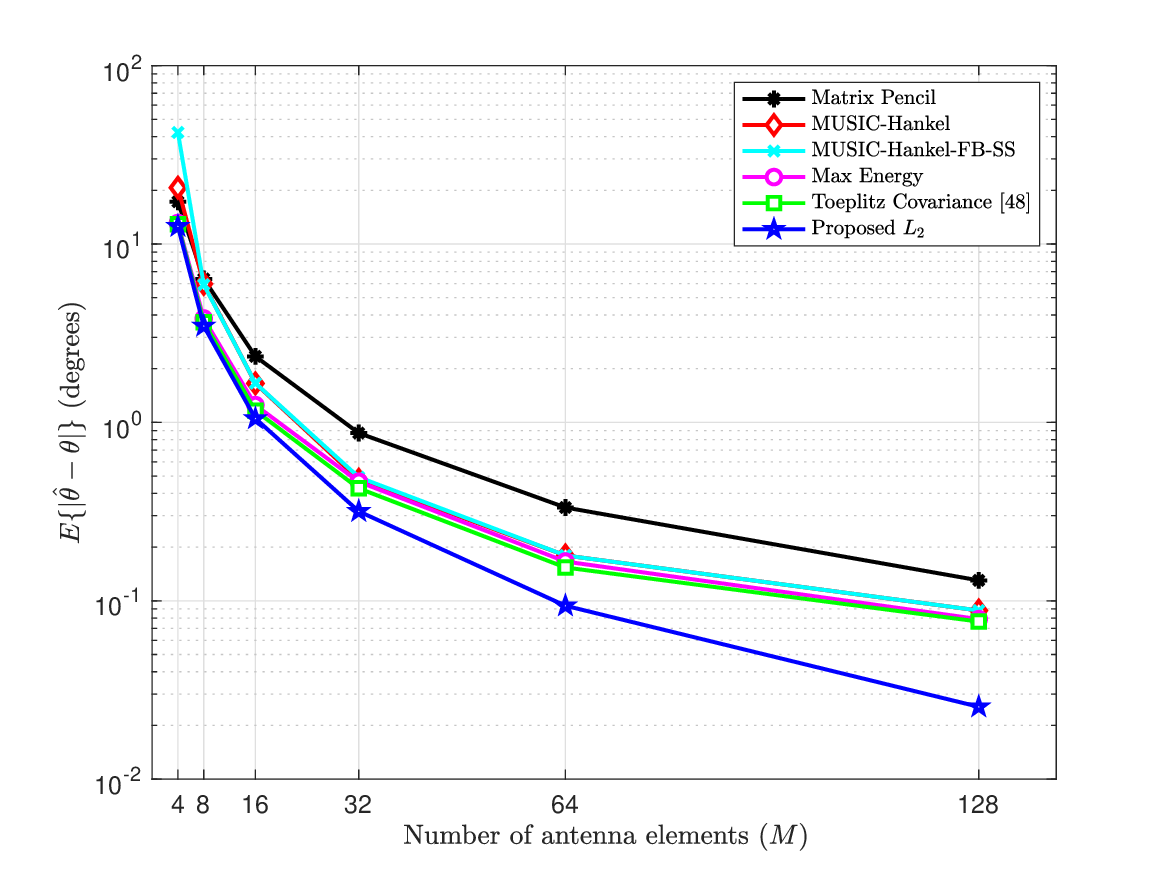}
            \label{fig:mae_snr_0}
        }
    \end{minipage}
    \hfill
    \begin{minipage}[t]{0.49\textwidth}
        \centering
        \subfloat[]{
            \includegraphics[width=\textwidth]{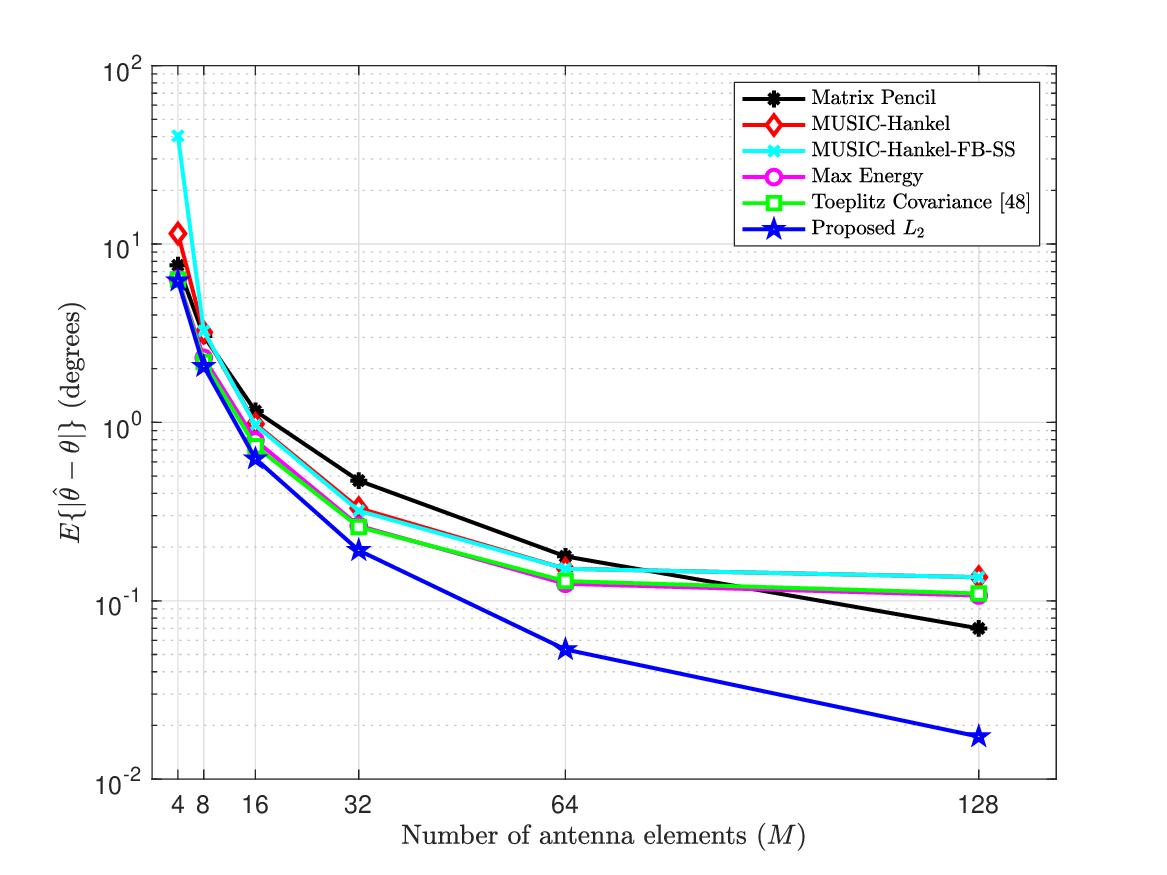}
            \label{fig:mae_snr_5}
        } \\
        \subfloat[]{
            \includegraphics[width=\textwidth]{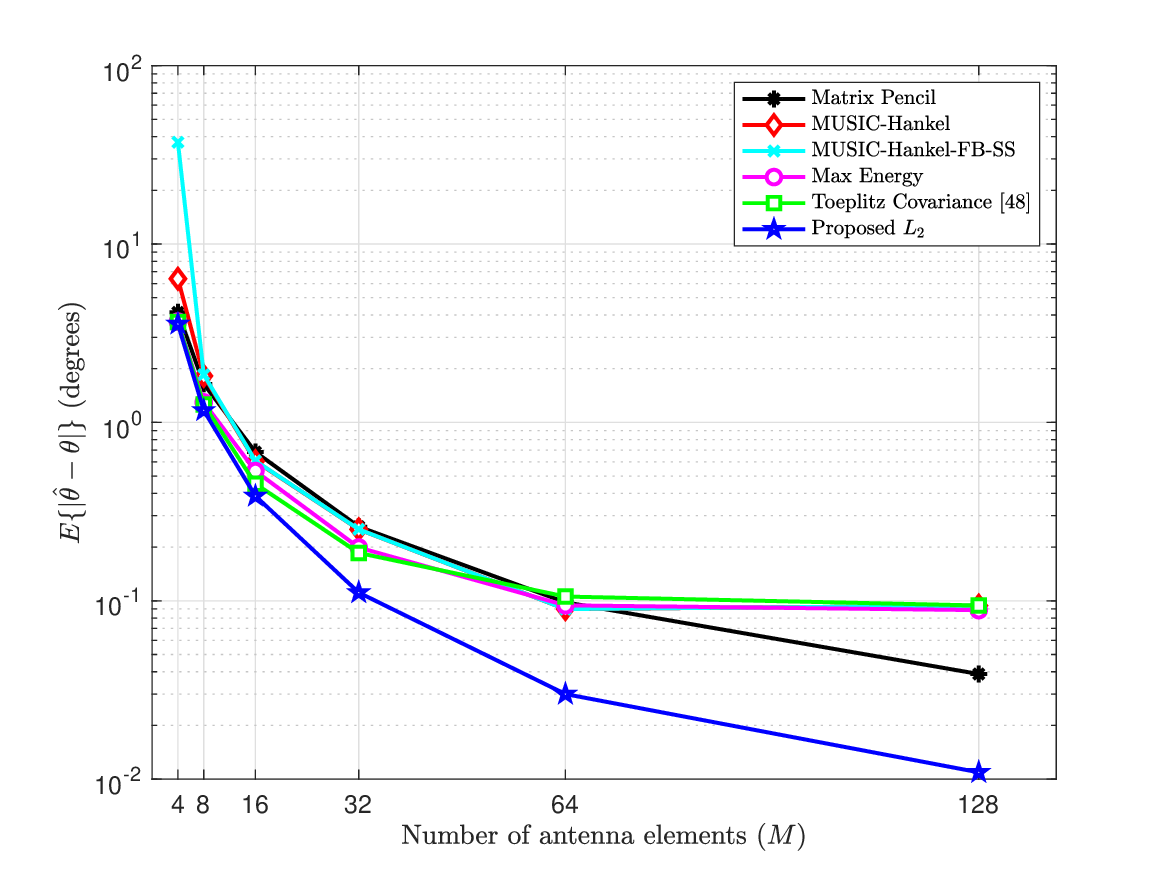}
            \label{fig:mae_snr_10}
        }
    \end{minipage}
    \caption{Mean absolute estimation error versus number of antenna array elements: (a) SNR = -5 dB, (b) SNR = 0 dB, (c) SNR = 5 dB, (d) SNR = 10 dB.}
    \label{fig:mae_all_wgn}
\end{figure*}

Next, we consider a challenging sensing environment where the noise exhibits impulsive behavior modeled as a Bernoulli–Gaussian mixture of the form 
\begin{equation}
\label{impulsive_noise_1}
\begin{split}
   \mathbf{n} \sim(1-p) \, \mathcal{C} \mathcal{N}\left(\mathbf{0}, \sigma_1^2 \mathbf{I}_{D \times D}\right)+p \, \mathcal{C N}\left(\mathbf{0}, \sigma_2^2 \mathbf{I}_{D \times D}\right) 
\end{split}
\end{equation}
where $p \in (0,1)$ is the impulse probability and $\sigma_2^2 \gg \sigma_1^2$. 
For fixed unknown signal amplitude $x \in \mathbb{C}$ across samples
$\mathbf{r}_i$, $i = 0, \ldots, W-1$, in (\ref{sliding_sub}), the signal-to-noise ratio defined now as $\text{SNR} \triangleq \frac{|x|^2}{(1-p)\sigma_1^2 + p\sigma_2^2}$dB
is controlled by setting $x = \sqrt{10^{SNR / 10}\left[(1-p) \sigma_1^2+p \sigma_2^2\right]} e^{j \phi}$ where $\phi$ is drawn from the uniform $[0, 2\pi]$ distribution.
We recall that under the assumption of isotropic i.i.d. Laplace noise the proposed rank-$1$ Hankel-decomposition DoA estimator in Fig. \ref{l1_grid_alg} under the $L_1$-norm is maximum-likelihood optimal. We anticipate, therefore, that the $L_1$-norm Hankel-decomposition estimator excels in the heavy-tailed Bernoulli-Gaussian noise environment under consideration.

Fig. \ref{figBGM:mae_all_0.1} reports the mean absolute estimation error, averaged over $2000$ independent experiments, of the proposed rank-$1$ Hankel decomposition method under the $L_1$-norm and all other estimators \emph{(i)-(v)}, for SNR levels  $-5$dB, $0$dB, $5$dB, and $10$dB. 
The impulse probability is set to $p=0.1$ and the low and high noise variances to $\sigma_1^2 = 1$ and $\sigma_2^2 = 200$.
While performance improves monotonically with increasing impulsive-noise SNR and number of antenna elements for all methods, the proposed DoA estimator outperforms significantly the competing approaches and the performance gap widens as the number of antennas and/or the SNR increases. Notably, at antenna array size $M=128$ and all studied SNR values there is about an order of magnitude improvement in absolute error between the developed estimator and the other methods.

\begin{figure*}[htbp]
    \centering
    \begin{minipage}[t]{0.49\textwidth}
        \centering
        \subfloat[]{
            \includegraphics[width=\textwidth]{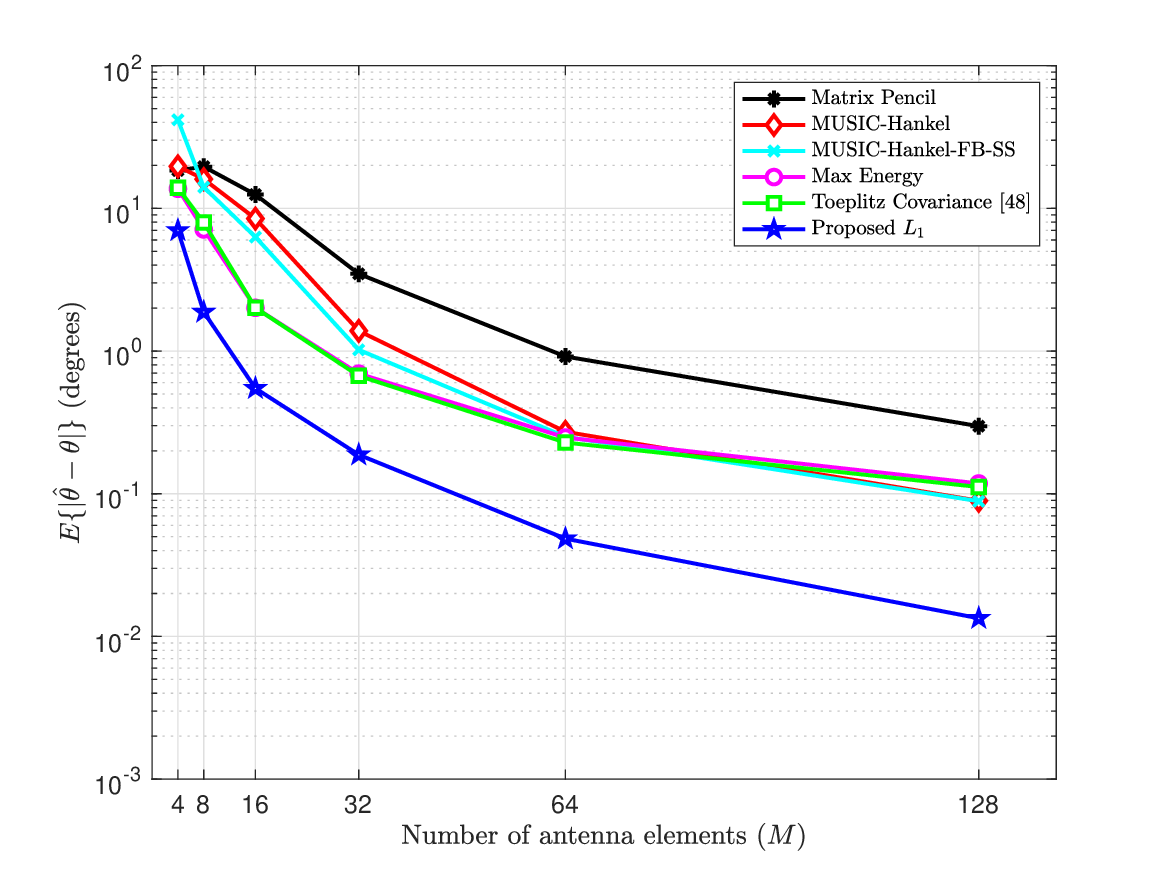}
            \label{figBGM:mae_snr_-5_0.1}
        } \\
        \subfloat[]{
            \includegraphics[width=\textwidth]{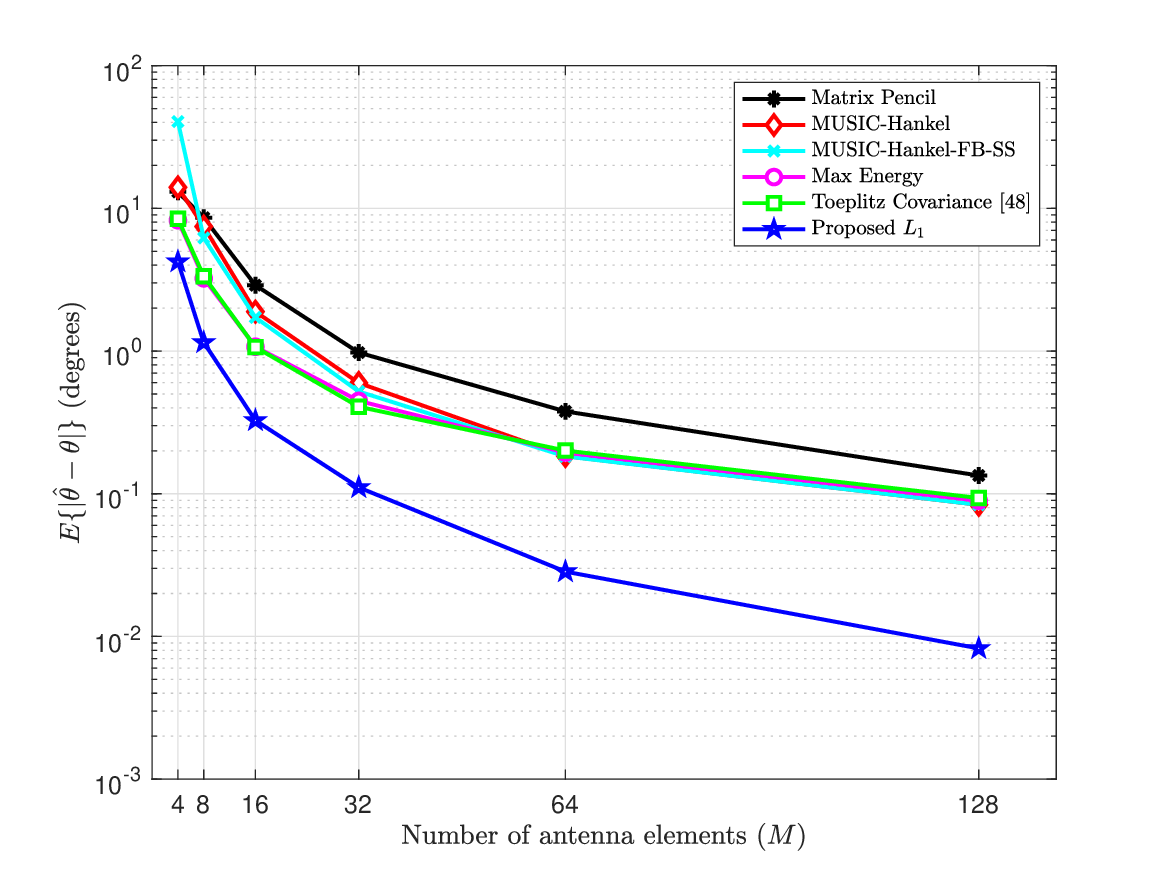}
            \label{figBGM:mae_snr_0_0.1}
        }
    \end{minipage}
    \hfill
    \begin{minipage}[t]{0.49\textwidth}
        \centering
        \subfloat[]{
            \includegraphics[width=\textwidth]{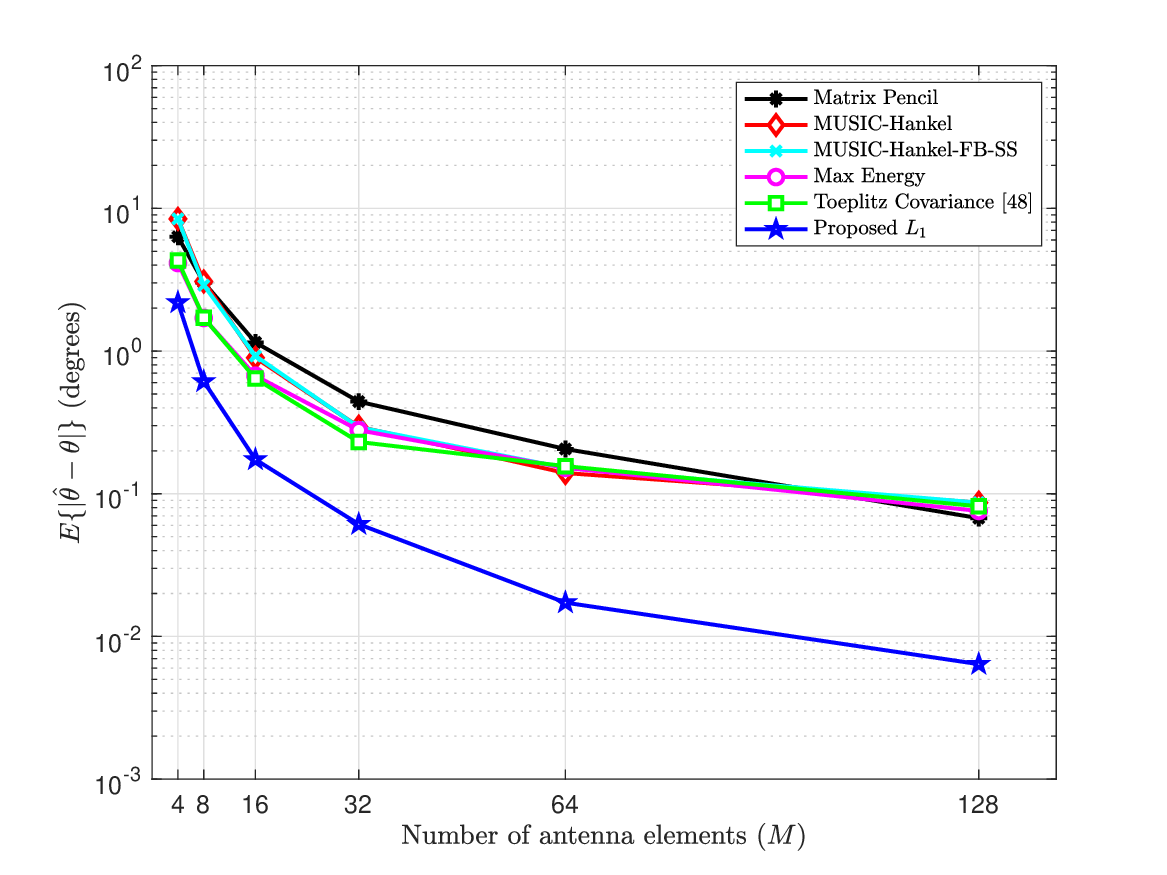}
            \label{figBGM:mae_snr_5_0.1}
        } \\
        \subfloat[]{
            \includegraphics[width=\textwidth]{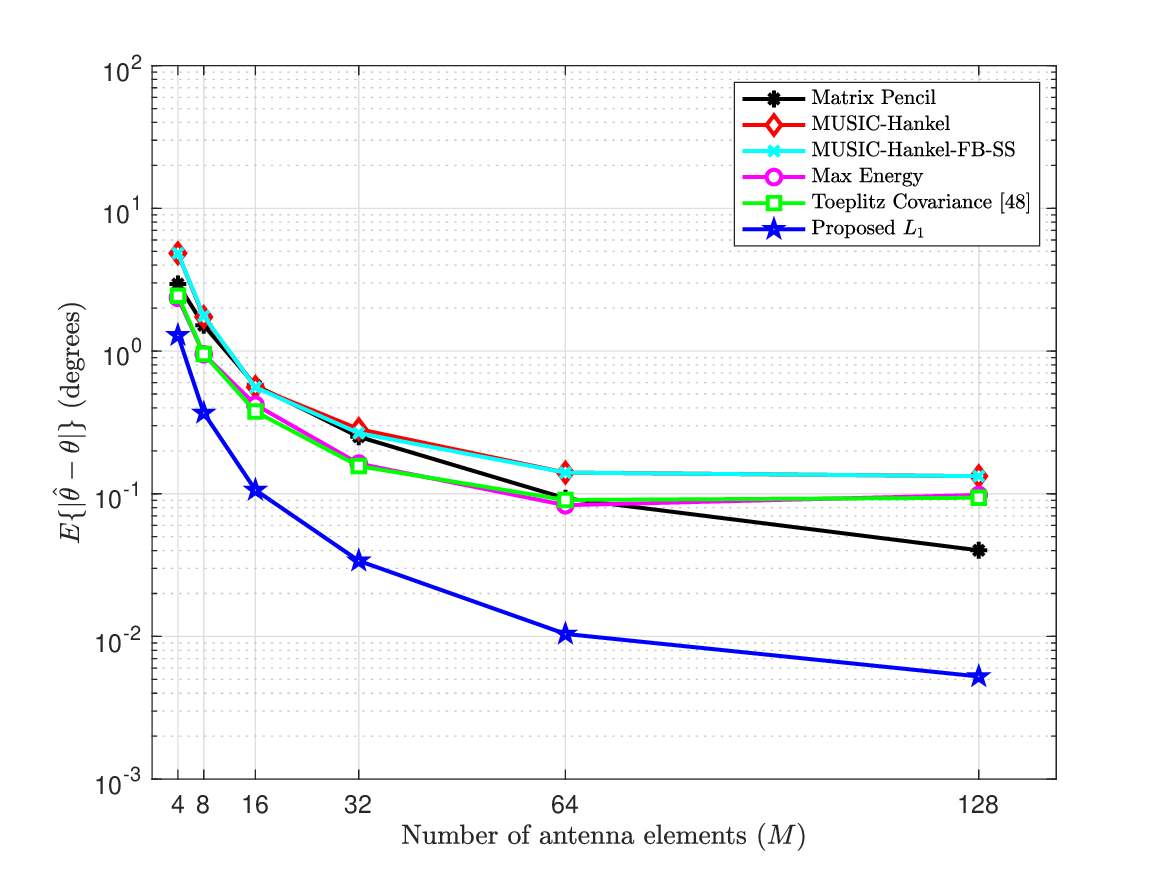}
            \label{figBGM:mae_snr_10_0.1}
        }
    \end{minipage}
    \caption{Mean absolute estimation error versus number of antenna array elements, impulse probability $p=0.1$: (a) SNR = -5 dB, (b) SNR = 0 dB, (c) SNR = 5 dB, (d) SNR = 10 dB.}
    \label{figBGM:mae_all_0.1}
\end{figure*}

Fig. \ref{figBGM:mae_all_0.2} repeats the study of Fig. \ref{figBGM:mae_all_0.1} under higher impulse probability $p=0.2$. The same overall conclusions hold true. The robustness of the developed rank-$1$ Hankel decomposition method under the $L_1$-norm is prevalent. The results underscore the sustained accuracy of the estimator in challenging impulsive-noise environments.

\begin{figure*}[htbp]
    \centering
    \begin{minipage}[t]{0.49\textwidth}
        \centering
        \subfloat[]{
            \includegraphics[width=\textwidth]{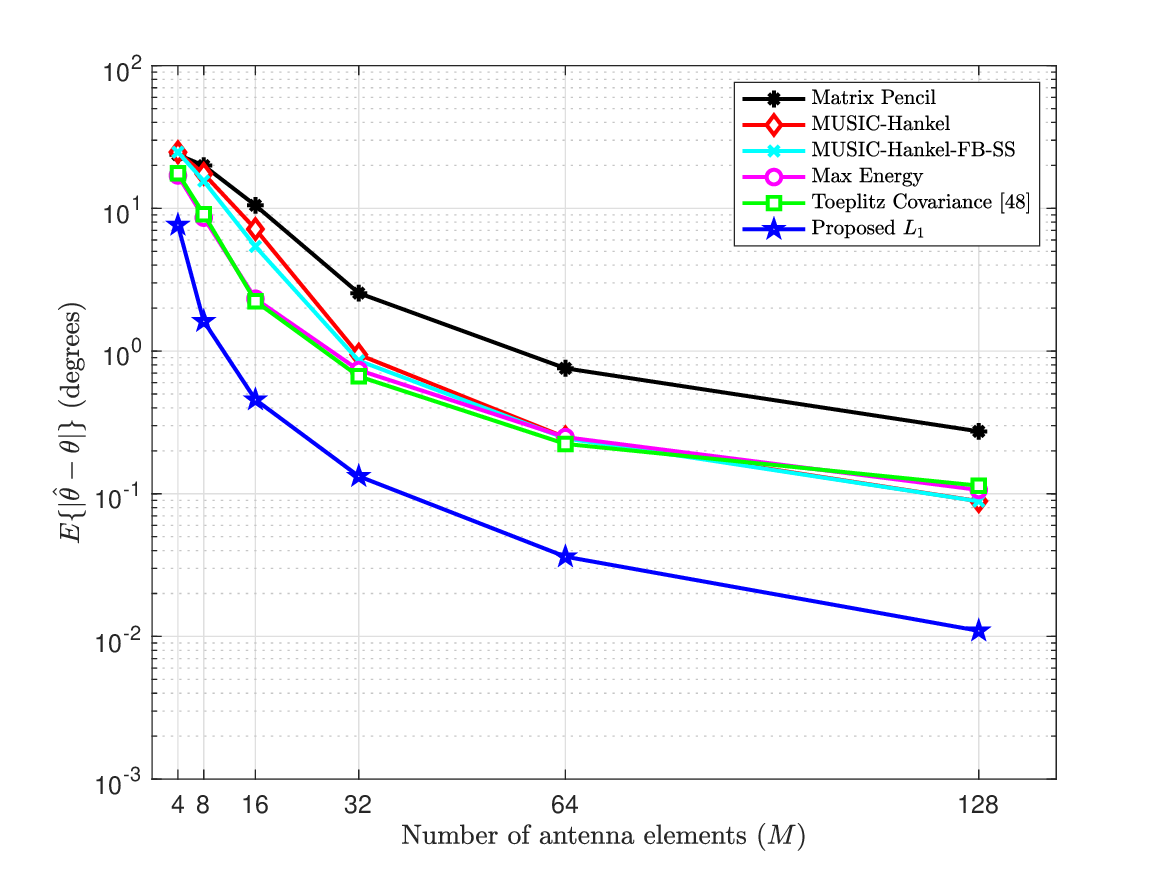}
            \label{figBGM:mae_snr_-5_0.2}
        } \\
        \subfloat[]{
            \includegraphics[width=\textwidth]{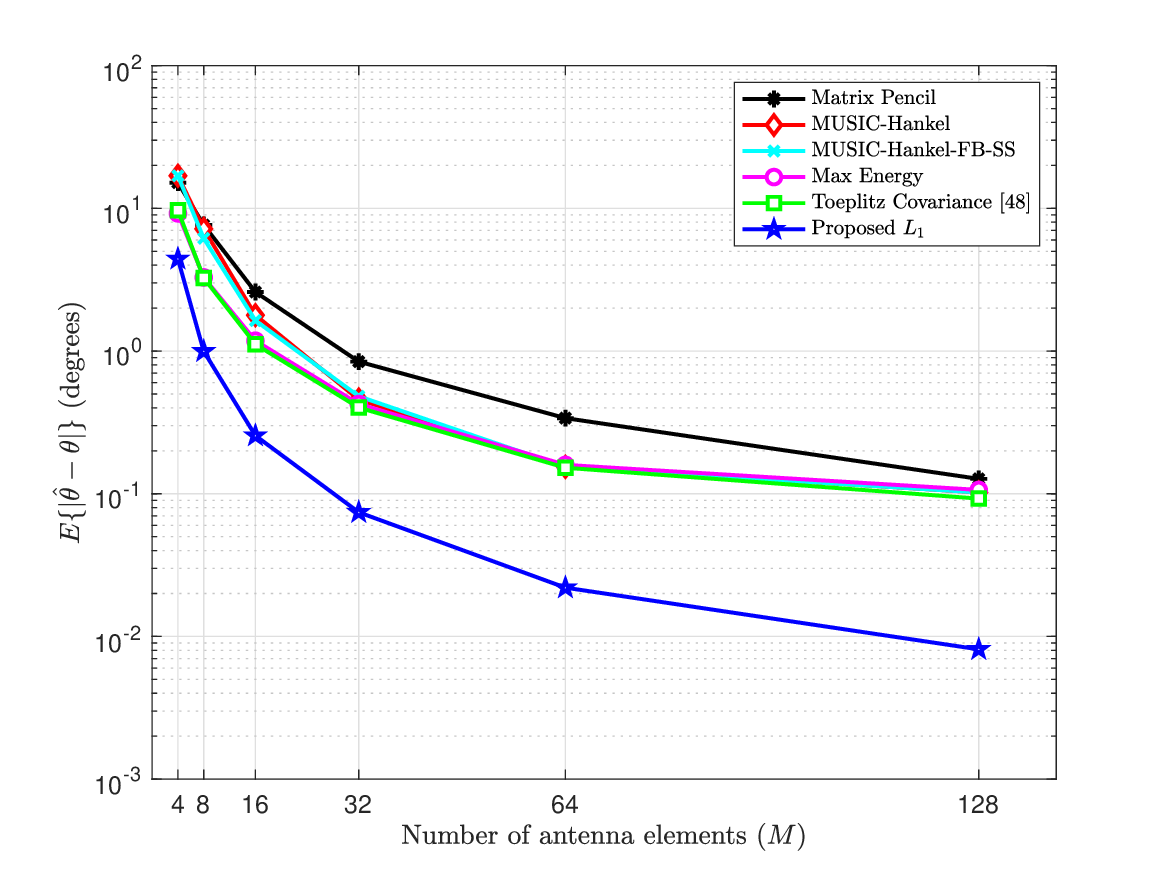}
            \label{figBGM:mae_snr_0_0.2}
        }
    \end{minipage}
    \hfill
    \begin{minipage}[t]{0.49\textwidth}
        \centering
        \subfloat[]{
            \includegraphics[width=\textwidth]{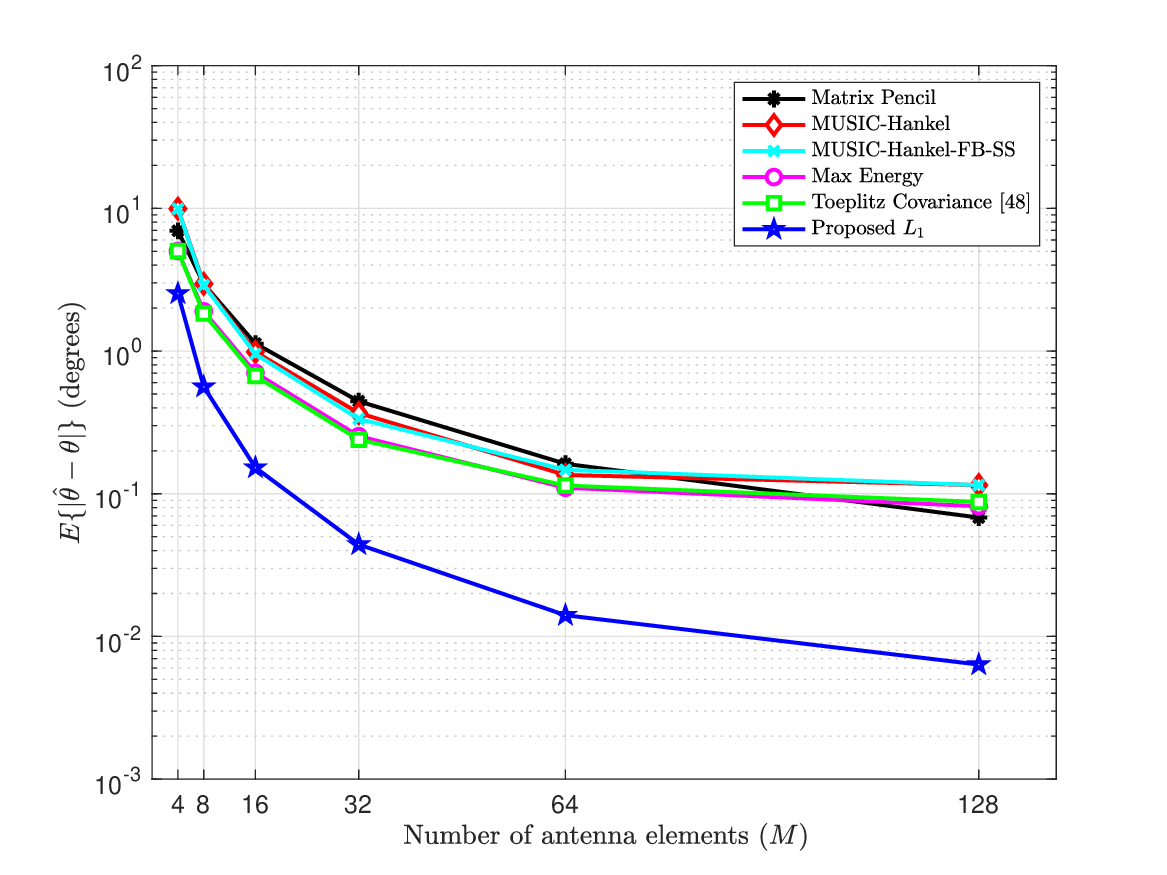}
            \label{figBGM:mae_snr_5_0.2}
        } \\
        \subfloat[]{
            \includegraphics[width=\textwidth]{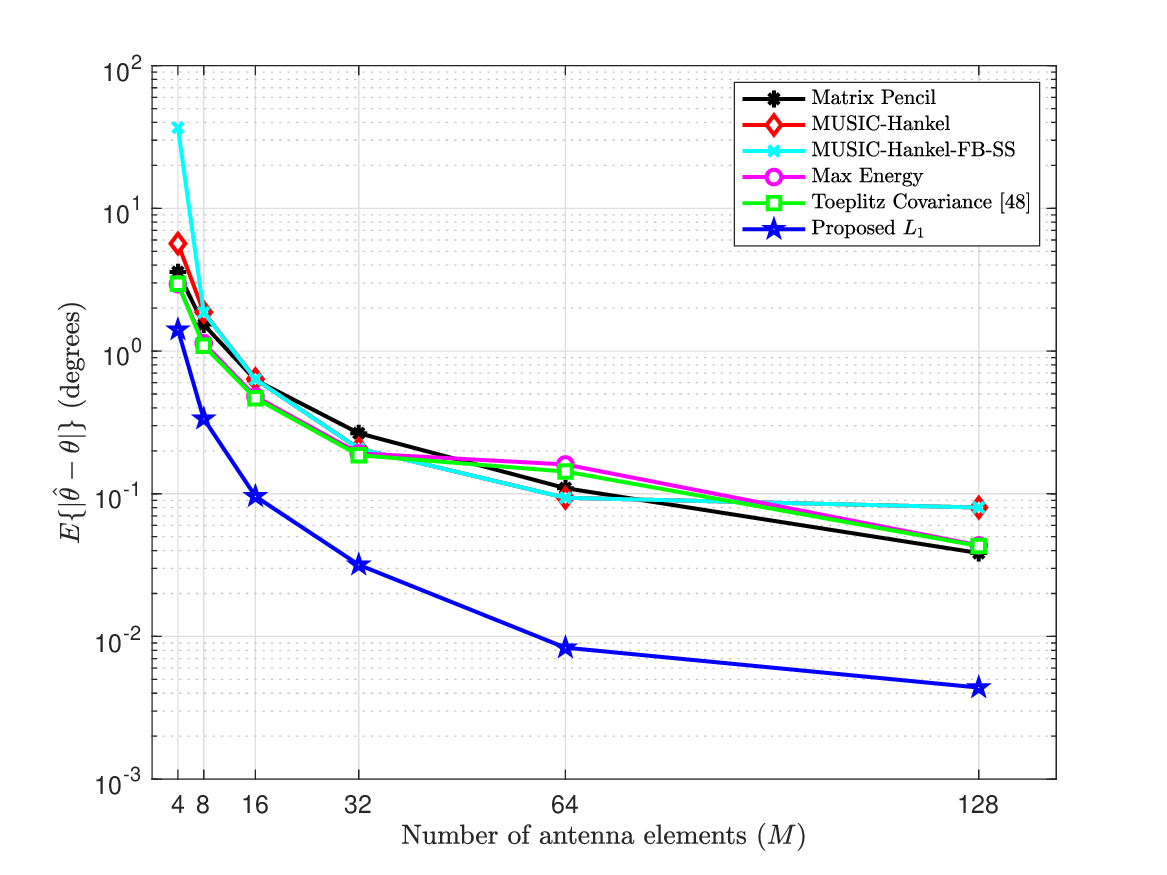}
            \label{figBGM:mae_snr_10_0.2}
        }
    \end{minipage}
    \caption{Mean absolute estimation error versus number of antenna array elements, impulse probability $p=0.2$: (a) SNR = -5 dB, (b) SNR = 0 dB, (c) SNR = 5 dB, (d) SNR = 10 dB.}
    \label{figBGM:mae_all_0.2}
\end{figure*}

Finally, to assess the practical robustness of the proposed $L_1$-norm rank-1 Hankel decomposition–based DoA estimator, we conduct experiments using the publicly available UAV dataset in \cite{10356267}. The dataset consists of measurements collected by a $5 \times 8$ uniform rectangular array (URA) mounted on a hovering drone. Under the sliding-window sensing architecture and $D=4$ available RF processing chains, the URA measurements are processed using overlapping subarrays as illustrated in Fig. \ref{fig:sliding_array_arch_uav}.
Due to imperfections in the dataset, three antenna elements are unavailable, resulting in an irregular effective array geometry. To emulate sensor faults, the samples corresponding to the missing elements are replaced with noise-only measurements with elevated power relative to the healthy sensors. Specifically, for a given a subarray data vector $\mathbf{r}_i, \quad i = 0, \ldots W-1$, the missing antenna array element values are replaced with i.i.d noise samples drawn from $\mathcal{C N}\left(0, \sigma_i^2\right)$ where 
\begin{equation}
\label{faulty_elements}
\sigma_i^2 \triangleq 10^{\alpha / 10} \frac{1}{\left|\Omega_i\right|} \sum_{j \in \Omega_i}\left|r_i[j]\right|^2 ,
\end{equation}
$\Omega_i$ is the set of indices corresponding to the healthy sensor entries of $\mathbf{r}_i$, and $\alpha$ controls the severity of the injected sensor fault in dB, representing the power increase relative to the available healthy sensors within the same subarray.


\begin{figure}[t]
    \centering
    \includegraphics[width=0.8
    \linewidth]{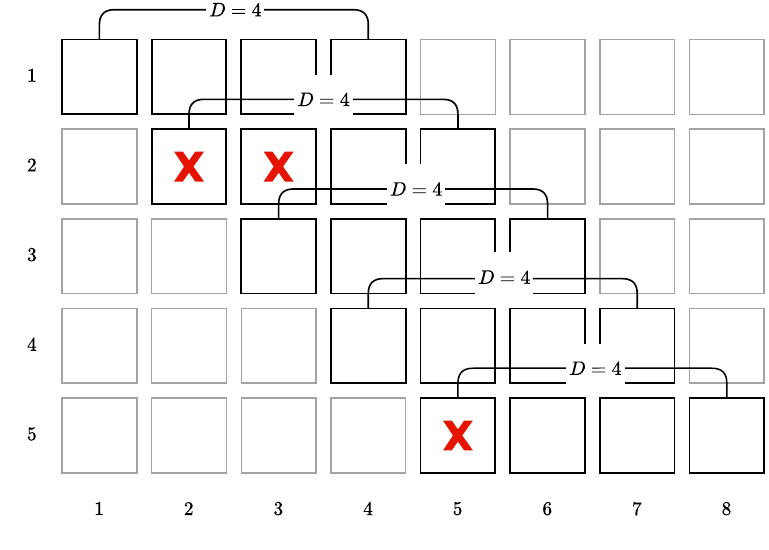}
    \caption{Sliding-window acquisition over a $5 \times 8$ URA with $D=4$ available RF processing chains. `X' markers indicate unavailable antenna element measurements, which lead to an irregular effective array geometry and corrupted sliding-window data \cite{10356267}.}
    \label{fig:sliding_array_arch_uav}
\end{figure}

Fig. \ref{fig:uav_results} reports the mean absolute DoA estimation error over $5000$ independent experiments for all frameworks under consideration for a UAV hovering location with an azimuth angle relative to the base station of $-9.64^{\circ}$ and sensor fault injection parameter $\alpha = 10$dB. The developed robust $L_1$-norm Hankel-decomposition DoA estimator sustains the lowest error among the evaluated methods, highlighting its resilience to imperfect array geometries and noise-only sensor feeds introduced by faulty antenna elements. These results further demonstrate that in real-world deployments where antenna failures of various forms may occur rather frequently, the $L_1$-norm Hankel-decomposition DoA estimator provides a leading option. 

\begin{figure}[t]
    \centering
    \includegraphics[width=\columnwidth]{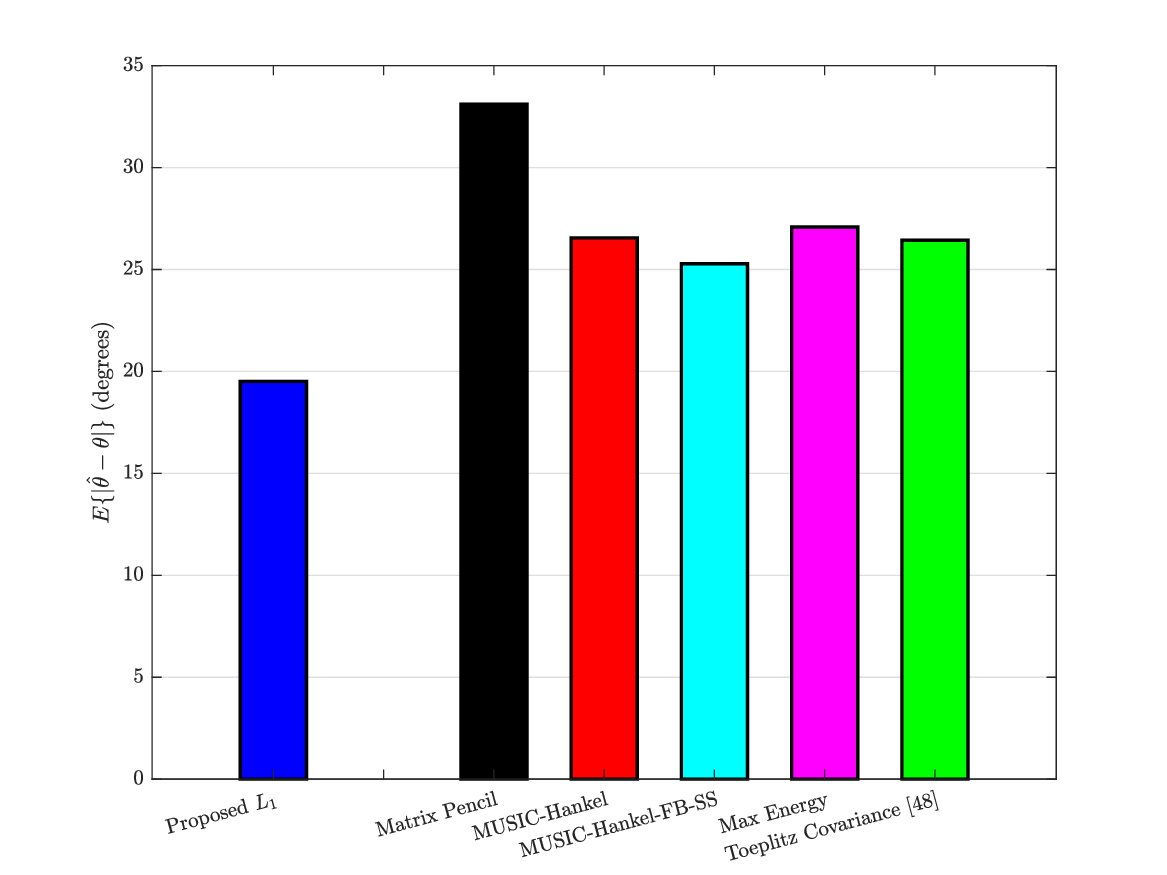}
    \caption{Mean absolute DoA estimation error for UAV hovering location with signal DoA $-9.64^{\circ}$.}
    \label{fig:uav_results}
\end{figure}

\section{Conclusions}
We addressed formally the problem of computing optimal rank-$1$ Hankel-structured (and Toeplitz-structured) decompositions of arbitrary complex matrices under both the $L_2$-norm and the $L_1$-norm error formulation and developed accurate and computationally tractable algorithms for both settings. In modern data acquisition protocols that impose naturally an underline Hankel structure to the sensed signal of interest in arbitrary noise, the developed algebraic algorithms are new prime tools for signal processing.

In this paper, we considered as an illustrative example space-series data analysis and, in particular, the problem of signal direction of arrival (DoA) estimation from few snapshots from large arrays with limited number of RF chains. We established that the $L_2$-norm Hankel-decomposition algorithm provides maximum-likelihood DoA estimates under white Gaussian noise. We established that the $L_1$-norm Hankel-decomposition algorithm provides maximum-likelihood DoA estimates under i.i.d. isotropic Laplace noise. Extensive simulation studies and experiments evaluated the performance of the two DoA estimators against leading alternatives in white Gaussian and heavy-tailed noise showing truly significant performance gains.

Collectively, these results demonstrated that structured rank-$1$ decompositions provide a powerful and principled foundation for a broad class of inference tasks arising in modern machine learning with fast DoA estimation representing a particularly impactful application in sensing environments characterized by non-stationarity, impulsive disturbances, and hardware calibration imperfections.

\appendix
\section*{Proof of Theorem \ref{ml_gaussian}}

\begin{IEEEproof}
Following the developments in (\ref{l2_rank1})--(\ref{c}) we rewrite 
\begin{equation}
\hat{\theta}_{L_2}=\underset{\theta \in\left[-90^{\circ}, 90^{\circ}\right)}{\operatorname{argmin}}\left\|\mathbf{X}-\hat{c}_{L_2}(z(\theta)) \mathbf{s}_D(z(\theta)) \mathbf{s}_W(z(\theta))^T\right\|_2^2
\end{equation} 
as 
\begin{equation}
\label{argmax_l2}
\hat{\theta}_{L_2}=\underset{\theta \in\left[-90^{\circ}, 90^{\circ}\right)}{\operatorname{argmax}}\left| \mathbf{s}_D(z(\theta))^{H} \mathbf{X} \mathbf{s}_W(z(\theta))^{*}\right|^2. 
\end{equation} 

Applying the identity
\begin{equation}
\mathrm{vec}(\mathbf{A}\mathbf{C}\mathbf{B}^T) = (\mathbf{B} \otimes \mathbf{A})\,\mathrm{vec}(\mathbf{C})
\end{equation}
to (\ref{argmax_l2}) with $\mathbf{A} = \mathbf{s}_D(z(\theta))^H$, $\mathbf{C} = \mathbf{X}$, and $\mathbf{B} = \mathbf{s}_W(z(\theta))^{*}$, we obtain
\begin{equation}
\label{vec_identity_exp}
\mathbf{s}_D(z(\theta))^{H} \mathbf{X} \mathbf{s}_W(z(\theta))^{*}
= (\mathbf{s}_W(z(\theta)) \otimes \mathbf{s}_D(z(\theta)))^{H} \mathrm{vec}(\mathbf{X}).
\end{equation}

Defining the vectorized observation 
\begin{equation}
    \mathbf{r}^v \triangleq \mathrm{vec}(\mathbf{X}) \in \mathbb{C}^{DW}
\end{equation}
and the virtual steering vector
\begin{equation}
\mathbf{a}^v_{DW}(z(\theta)) \triangleq \mathbf{s}_W(z(\theta)) \otimes \mathbf{s}_D(z(\theta)) \in \mathbb{C}^{DW}, 
\end{equation}
 we obtain by (\ref{vec_identity_exp})
\begin{equation}
\label{virtual_sample_ml}
\hat{\theta}_{L_2}
=\underset{\theta \in\left[-90^{\circ}, 90^{\circ}\right)}{\operatorname{argmax}}
\left|\mathbf{a}^v_{DW}(\theta)^H\mathbf{r}^v\right|^2.
\end{equation} 

Under the signal model in (\ref{signal_model_1_signal}) and sensing model in (\ref{sliding_sub}), (\ref{agg_meas_matrix}) the aggregate measurement matrix satisfies
\begin{equation}
\mathbf{X} = x\, \mathbf{s}_D(z(\theta_0)) \mathbf{s}_W(z(\theta_0))^T + \mathbf{N}
\end{equation}
where $\mathbf{N}$ consists of i.i.d.\ complex Gaussian noise samples. Vectorizing yields
\begin{equation}
\mathbf{r}^v = x\, \mathbf{a}^v_{DW}(\theta_0) + \mathbf{n}^v,
\end{equation}
with $\mathbf{n}^v \sim \mathcal{CN}(\mathbf{0}, \sigma^2 \mathbf{I})$.
For this model, the deterministic maximum-likelihood estimator of $\theta$ is
\begin{equation}
\hat{\theta}_{\mathrm{ML}}
= \underset{\theta \in [-90^\circ,90^\circ)}{\arg\max}
\left| \mathbf{a}^v_{DW}(\theta)^H \mathbf{r}^v \right|^2.
\end{equation}
Comparing with (\ref{virtual_sample_ml}), we conclude that
$$
\hat{\theta}_{L_2} = \hat{\theta}_{\mathrm{ML}},
$$
which establishes that the proposed $L_2$-norm rank-$1$ Hankel estimator is maximum-likelihood optimal under i.i.d.\ complex Gaussian noise.
\end{IEEEproof}

\section*{Proof of Theorem \ref{ml_laplace}}
\begin{IEEEproof}
For observations of the form 
\begin{equation}
\mathbf{x}_i = v_i\boldsymbol{\mu} + \mathbf{n}_i, \quad i = 0,\ldots,W-1,
\end{equation}
collected in the presence of i.i.d. complex Laplace noise arranged in $\mathbf{X} \in \mathbb{C}^{D \times W} $, the maximum-likelihood estimator is given by 
\begin{equation}
(\hat{\boldsymbol{\mu}},\hat{\mathbf{v}})_{\mathrm{ML}} =\underset{\boldsymbol{\mu},\,\mathbf{v}}{\operatorname{argmin}}\|\mathbf{X} - \boldsymbol{\mu}\mathbf{v}^T\|_1.
\end{equation}
Considering the sliding-subarray acquisition model in (\ref{sliding_sub}) and factoring out $z(\theta)^i$ from the $i$-th subarray we rewrite
\begin{equation}
\mathbf{r}_i = z(\theta)^i \, x\, \mathbf{s}_D(z(\theta)) + \mathbf{n}_i, \qquad i=0,\ldots,W-1,
\end{equation}
so that $v_i = z(\theta)^i$ and $\boldsymbol{\mu} = x\,\mathbf{s}_D(z(\theta))$ is fixed across subarrays. Under this parameterization, the aggregate measurement matrix from (\ref{agg_meas_matrix}) is of the form $\mathbf{X} = \boldsymbol{\mu}\mathbf{v}^T + \mathbf{N}$ with $\mathbf{v} = \mathbf{s}_W(z(\theta))$. The joint maximum-likelihood estimate becomes
\begin{equation}
(\hat{x}, \hat{\theta})_{\mathrm{ML}} = \underset{x \in \mathbb{C}, \, \theta \in\left[-90^{\circ}, 90^{\circ}\right)}{\operatorname{argmin}} \left\|\mathbf{X} - x\,\mathbf{s}_D(z(\theta))\,\mathbf{s}_W(z(\theta))^T\right\|_1.
\end{equation}
Finally, for each fixed $\theta$, optimization over $x$ yields $\hat{c}_{L_1}(z(\theta))$ as defined in (\ref{c_l1}), which gives exactly the estimator in (\ref{l1_rank1_doa}),
\begin{equation}
\hat{\theta}_{L_1} =\underset{\theta \in\left[-90^{\circ}, 90^{\circ}\right)}{\operatorname{argmin}}\left\|\mathbf{X}-\hat{c}_{L_1}(z(\theta)) \mathbf{s}_D(z(\theta)) \mathbf{s}_W(z(\theta))^T\right\|_1.
\end{equation} 
\end{IEEEproof}

\bibliographystyle{IEEEtran}
\bibliography{IEEEabrv,bibliography}

\end{document}